\documentclass[twocolumn, switch]{article} % Method A for two-column formatting

\usepackage{preprint}
\usepackage{longtable} 
\usepackage{array}
\usepackage{amsmath, amsthm, amssymb, amsfonts}

\usepackage[numbers,square]{natbib}
\usepackage[utf8]{inputenc}	% allow utf-8 input
\usepackage[T1]{fontenc}	% use 8-bit T1 fonts
\usepackage{xcolor}		% colors for hyperlinks
\usepackage[colorlinks = true,
            linkcolor = purple,
            urlcolor  = blue,
            citecolor = cyan,
            anchorcolor = black]{hyperref}	% Color links to references, figures, etc.
\usepackage{booktabs} 		% professional-quality tables
\usepackage{nicefrac}		% compact symbols for 1/2, etc.
\usepackage{microtype}		% microtypography
\usepackage{lineno}		% Line numbers
\usepackage{float}			% Allows for figures within multicol

\usepackage{lipsum}		%  Filler text

\usepackage{newfloat}
\DeclareFloatingEnvironment[name={Supplementary Figure}]{suppfigure}
\usepackage{sidecap}
\sidecaptionvpos{figure}{c}

\usepackage{titlesec}
\titlespacing\section{0pt}{12pt plus 3pt minus 3pt}{1pt plus 1pt minus 1pt}
\titlespacing\subsection{0pt}{10pt plus 3pt minus 3pt}{1pt plus 1pt minus 1pt}
\titlespacing\subsubsection{0pt}{8pt plus 3pt minus 3pt}{1pt plus 1pt minus 1pt}

\usepackage{tikz,xcolor,hyperref}

\definecolor{lime}{HTML}{A6CE39}
\DeclareRobustCommand{\orcidicon}{
	\begin{tikzpicture}
	\draw[lime, fill=lime] (0,0) 
	circle [radius=0.16] 
	node[white] {{\fontfamily{qag}\selectfont \tiny ID}};
	\draw[white, fill=white] (-0.0625,0.095) 
	circle [radius=0.007];
	\end{tikzpicture}
	\hspace{-2mm}
}
\foreach \x in {A, ..., Z}{\expandafter\xdef\csname orcid\x\endcsname{\noexpand\href{https://orcid.org/\csname orcidauthor\x\endcsname}
			{\noexpand\orcidicon}}
}
\title{Vector Symbolic Policy Gradient}
\newcommand\titleheader{Vector Symbolic Policy Gradient}
\usepackage{amsthm}
\newtheorem{proposition}{Proposition}
\newtheorem{lemma}{Lemma}
\newtheorem{corollary}{Corollary}
\theoremstyle{remark}
\newtheorem{remark}{Remark}

\usepackage{xwatermark}
\usepackage{mathtools}
\usepackage{amssymb}
\usepackage{algorithm}
\usepackage{algorithmic}
\usepackage{newfloat}
\usepackage{authblk}

\author[1]{Ryozo Masukawa}
\author[1]{Sanggeon Yun}
\author[1]{SungHeon Jeong}
\author[1]{Hyunwoo Oh}
\author[1]{Raheeb Hassan}
\author[2]{Pietro Mercati}
\author[3]{Nathaniel D. Bastian}
\author[4]{Mahdi Imani}
\author[1]{Mohsen Imani}

\affil[1]{University of California, Irvine}
\affil[2]{Intel Corporation}
\affil[3]{Johns Hopkins University}
\affil[4]{Northeastern University}

\begin{document}

\twocolumn[ % Method A for two-column formatting
  \begin{@twocolumnfalse} % Method A for two-column formatting
  
\maketitle

%\keywords{First keyword \and Second keyword \and More} % (optional)
\vspace{0.35cm}

  \end{@twocolumnfalse} % Method A for two-column formatting
] % Method A for two-column formatting

%\begin{multicols}{2} % Method B for two-column formatting (doesn't play well with line numbers), comment out if using method A

%%%%%%%%%%%%%%%  Main text   %%%%%%%%%%%%%%%
% \linenumbers

%%%%%%%%%%%% Supplementary Methods %%%%%%%%%%%%
%\footnotesize
%\section*{Methods}

%%%%%%%%%%%%% Acknowledgements %%%%%%%%%%%%%
%\footnotesize
%\section*{Acknowledgements}

%%%%%%%%%%%%%%   Bibliography   %%%%%%%%%%%%%%
\begin{abstract}
% Hyperdimensional computing (HDC) offers lightweight distributed representations with simple algebraic learning rules and natural tolerance to noisy memory, but its role in discrete-action policy-gradient reinforcement learning (RL) remains underexplored. We introduce Vector-Symbolic Policy Gradient (VSPG), a categorical actor that stores one unit-norm hypervector per action and scores actions by similarity to an encoded state. Under the standard softmax policy-gradient surrogate, we show that the exact actor update is advantage-weighted hypervector bundling followed by row-wise normalization: visited state hypervectors are added to selected action memories and subtracted from competing memories according to their current probabilities. We further show that trained action hypervectors form compressed kernel memories, storing finite advantage-weighted kernel expansions over experience while retaining fixed-size inference cost. This view explains how encoder-induced similarity can transfer advantage evidence across nearby states without enumerating past samples at deployment. We evaluate VSPG on classic control, MiniGrid, and multi-agent SustainGym building-control tasks against DNN, linear, QHD, and matched PPO-style baselines. Across these benchmarks, VSPG achieves competitive final performance and favorable sample-efficiency trends. Under post-training quantization and random bit-flip corruption of stored actor parameters, VSPG degrades more gracefully than neural and linear actors, supporting its use in resource-constrained RL on unreliable edge and embedded systems.
Vector Symbolic Architecture (VSA) is built around a simple idea: distributed memories can be learned through lightweight algebra and remain useful even when their bits are unreliable. Yet this perspective has rarely been connected directly to discrete-action policy gradients. We introduce Vector-Symbolic Policy Gradient (VSPG), a categorical actor that represents each action by a unit-norm hypervector and chooses actions by similarity to an encoded state. 
We show that the standard softmax policy-gradient step has an exact vector-symbolic interpretation: advantage-weighted state hypervectors are bundled into the selected action memory and suppressed in competing memories, followed by row-wise normalization, so the actor trains in closed form with no optimizer state and logits bounded by construction.
Over training, these memories become fixed-size compressed kernel expansions that transfer advantage evidence across similar states without retaining past samples at inference. On classic control, MiniGrid, and multi-agent SustainGym, VSPG achieves competitive returns with favorable learning speed. Its distributed action memories also degrade substantially more gracefully than neural and linear actors under post-training quantization and random bit flips, making VSPG a promising actor for unreliable edge systems. An anonymized code is available \href{https://github.com/BiasLabProjects/VSPG.git}{here}.
\end{abstract}

\section{Introduction}

Vector Symbolic Architecture (VSA), also known as Hyperdimensional computing (HDC), is a brain-inspired computing paradigm rooted in theories of distributed
representation from cognitive science~\cite{kanerva2009hyperdimensional,fhrr}. VSA represents data using
high-dimensional distributed vectors, or hypervectors, and performs computation
through simple operations such as similarity search, bundling, binding, and
normalization.

Two properties are crucial for using VSA as a learning representation. First,
independently generated hypervectors are nearly orthogonal in high dimension,
allowing many pieces of information to be superposed in a single memory with
limited interference and supporting graceful degradation under noise~\cite{evidence_quasi_orthogonal,gorban2018blessing}. Second, an VSA encoder
$\phi$ maps inputs into an explicit high-dimensional inner-product space, where
$\phi(x)^\top \phi(x')$ can be designed to preserve a meaningful similarity
between inputs. This gives VSA encoders a theoretically grounded connection to
kernel-approximation methods~\cite{thomas2021theoretical,nyshd}. Together,
these properties have motivated the use of VSA in lightweight and
resource-constrained learning systems, with recent work spanning intelligent
sensing, hardware acceleration of IoT devices~\cite{hdc_application_1_fpga,hdc_application2_fpga,hdc_application3_fpga,
hdc_application4_iot,hdc_application_5_asic}.
\begin{figure}[t]
    \centering
    \includegraphics[width=\linewidth]{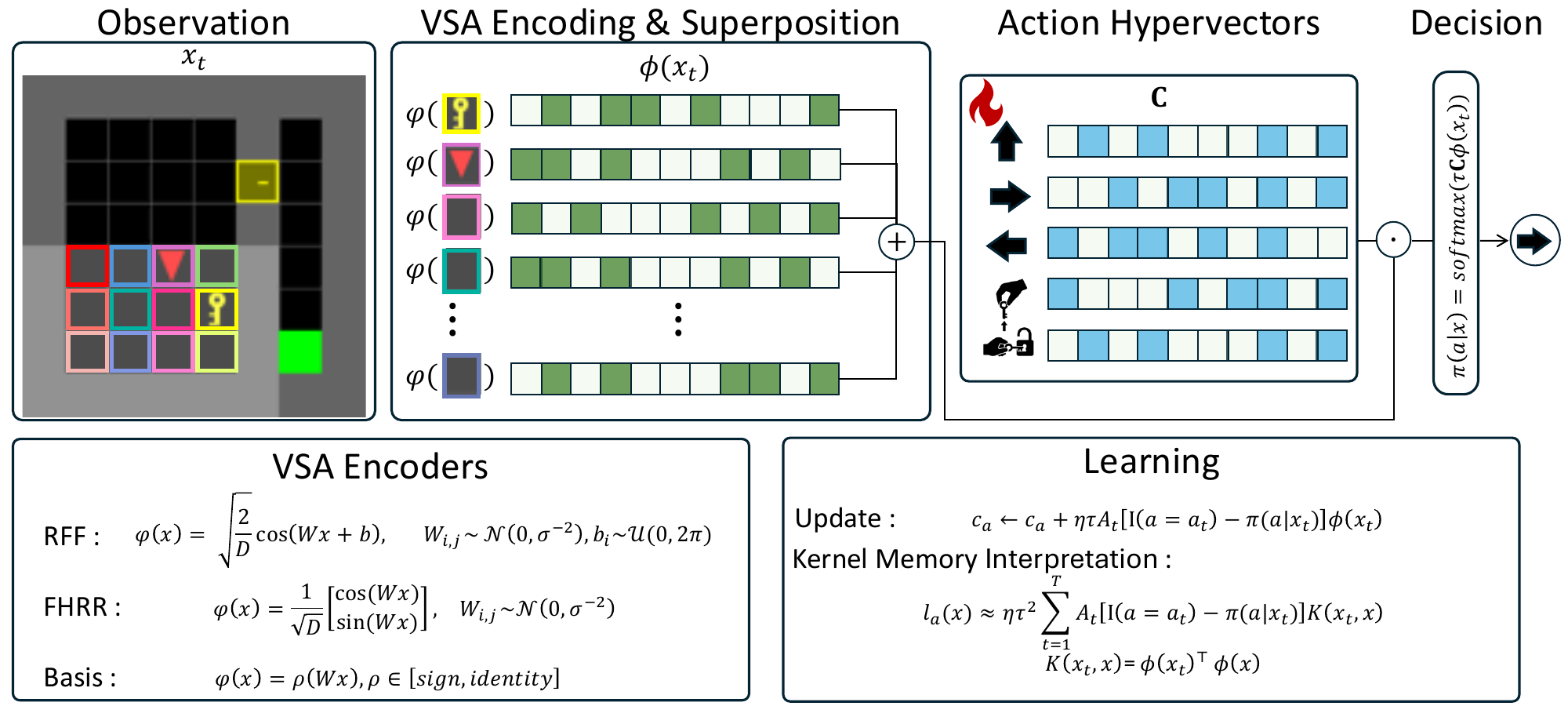}
%    \caption{Overview of VSPG. A vector-symbolic policy implements kernel-memory reinforcement learning.}
    \caption{Overview of VSPG. A fixed encoder maps the observation to a hypervector, action hypervectors score it by inner products, and the exact policy-gradient step bundles the encoded state into the kernel action memories.}
    \label{fig:teaser}
\end{figure}

% These properties are particularly relevant to deployment robustness in
% Reinforcement Learning (RL), where practical agents may run under
% low-precision execution or unreliable memory, and bit-level corruption of
% stored policy parameters can degrade autonomous decision making
% ~\cite{dulac2021challenges,bit_flip_1_rl,bit_flip_6_rl}. VSA provides a natural
% representation-level substrate for lightweight RL policies exposed to
% quantization and model-state faults. Existing VSA-based RL methods have mostly
% used hypervectors at the algorithmic and application level, as Q-value
% approximators trained by Bellman regression~\cite{qhd,bqhd} or as components
% of continuous actor--critic control~\cite{hdpg,actor_critic} but case theoretically underexplored: can action
% hypervectors directly parameterize a categorical actor, and can its
% policy-gradient update be written as a vector-symbolic memory operation?
These properties are particularly relevant to deployment robustness in
reinforcement learning (RL), where practical agents may run under
low-precision execution or unreliable memory, and bit-level corruption of
stored policy parameters can degrade autonomous decision making
~\cite{dulac2021challenges,bit_flip_1_rl,bit_flip_6_rl}. VSA provides a natural
representation-level substrate for lightweight RL policies exposed to
quantization and model-state faults. Existing VSA-based RL methods have mostly
used hypervectors at the algorithmic and application level, either as Q-value
approximators trained by Bellman regression~\cite{qhd,bqhd} or as components
of continuous actor--critic control~\cite{hdpg,actor_critic}, but leave open a basic question: can action hypervectors directly parameterize a categorical actor, and can its policy-gradient update be written as a vector-symbolic memory operation?
%but leave open a
%theoretically simple yet underexplored question: can action hypervectors
%directly parameterize a categorical actor, and can its policy-gradient update
%be written as a vector-symbolic memory operation?

\providecommand{\cmk}{\checkmark}
\providecommand{\emk}{$\bullet$}

\begin{table*}[ht]
\centering
\begin{minipage}{\linewidth}
\centering

\caption{Positioning of VSPG among the method families it builds on.}
\label{tab:positioning}

\small
\setlength{\tabcolsep}{5pt}

\resizebox{\linewidth}{!}{%
\begin{tabular}{@{}llccccc@{}}
\toprule
Method & Setting
& \shortstack{Closed-form\\update}
& \shortstack{Exact PG\\identity}
& \shortstack{Kernel\\expansion}
& \shortstack{Fixed-size\\inference}
& \shortstack{Bit-flip\\robustness} \\
\midrule

DNN softmax actor~\cite{ppo,mappo}
& disc./cont., PG & -- & -- & -- & \cmk & -- \\

Log-linear softmax PG~\cite{sutton1999policy,mei2020global}
& disc., PG & \cmk & \cmk & -- & \cmk & -- \\

Kernel (RKHS) policy search~
\cite{kernel3_pmlr-v38-lever15,kernel4_zhang2025residual}
& cont., PG & \cmk & \cmk & \cmk & -- & -- \\

\addlinespace[2pt]

QHD and extensions~\cite{qhd,bqhd}
& disc., value & \cmk & -- & -- & \cmk & -- \\

NavHD~\cite{lee2025navhd}
& disc., value & -- & -- & -- & \cmk & -- \\

HDPG~\cite{hdpg}
& cont., PG & -- & -- & -- & \cmk & -- \\

\addlinespace[2pt]

Fault-aware robust RL~\cite{bit_flip_1_rl,bit_flip_6_rl}
& disc., value & -- & -- & -- & \cmk & \cmk \\

\midrule

VSPG (ours)
& disc., PG & \cmk & \cmk & \cmk & \cmk & \cmk \\

\bottomrule
\end{tabular}%
}
\end{minipage}
\end{table*}

We answer this question with \emph{Vector-Symbolic Policy Gradient} (VSPG), a discrete-action actor that represents each action by a unit-norm hypervector and scores it by similarity to the encoded state. Under the standard softmax policy-gradient surrogate, we prove that its update is exactly advantage-weighted hypervector bundling followed by normalization, and therefore supports standard advantage estimators~\cite{gae,mappo}. We further show that each trained action hypervector is a fixed-size compressed kernel memory, storing an advantage-weighted kernel expansion over visited states and transferring evidence according to the encoder-induced similarity. This provides a concrete mechanism that can support sample-efficient learning without increasing inference-time memory. Finally, for bipolar action memories, we prove that greedy action selection is stable under random bit flips, with failure probability decaying exponentially in the hypervector dimension. VSPG thus connects VSA action memories, log-linear policy gradients, and kernel policy search while providing a quantitative robustness guarantee.

{
% We evaluate VSPG on classic control, MiniGrid~\cite{gymnasium,minigrid},
% and multi-agent SustainGym building control~\cite{yeh2023sustaingym} comparing against DNN, linear actor, and QHD~\cite{qhd} as a strong value-based VSA baseline for discrete-action
% RL. Across these benchmarks, VSPG achieves competitive final performance and superior sample-efficiency trends than matched actor baselines. We further evaluate robustness to bit-flip corruption in the stored action hypervectors, directly testing whether the distributed action-memory representation degrades gracefully under model-state faults.
We evaluate VSPG on classic control, MiniGrid~\cite{gymnasium,minigrid},
and multi-agent SustainGym building control~\cite{yeh2023sustaingym},
against DNN and linear actor baselines, as well as QHD~\cite{qhd}, a strong
value-based VSA method for discrete-action RL. Across these benchmarks, VSPG
achieves competitive final performance and stronger sample efficiency than
matched actor baselines. Under post-training bit-flip corruption, its
distributed action memories also retain performance substantially better than
both DNN and raw linear actors, demonstrating graceful degradation under
model-state faults.

Our contributions are as follows.
\begin{itemize}
\item We introduce VSPG, a discrete-action policy-gradient actor that represents each action with a unit-norm hypervector, and show that its update admits an exact vector-symbolic interpretation as advantage-weighted bundling followed by row normalization.

\item We show that trained action hypervectors form fixed-size compressed kernel memories and prove robustness of bipolar action memories to random bit flips.

\item We evaluate VSPG on classic control, MiniGrid, and multi-agent building control against neural, linear and value-based VSA, demonstrating competitive performance, favorable sample efficiency, and robustness to quantization and bit-level faults.
\end{itemize}

\section{Related Work}

\subsection{VSA and Kernel Policies for RL}
Prior VSA-RL methods mainly use hypervectors as lightweight function
approximators within existing RL formulations. HDPG~\cite{hdpg} addresses
continuous control with VSA-based Gaussian actors and critics, whereas
discrete-action methods largely follow the value-based QHD
framework~\cite{qhd,bqhd}, learning action-specific $Q$-value hypervectors
through Bellman-error-weighted bundling. This value-based line has also been
applied to cybersecurity, robotics, navigation, and sensing
applications~\cite{cyberrl,reacthd,lee2025navhd,hdflrobot}. VSPG instead
directly parameterizes a categorical softmax policy with action hypervectors,
extending vector-symbolic learning from value approximation to discrete-action
policy gradients. VSA methods have also demonstrated favorable data efficiency
in single-pass and online learning~\cite{onlinehd,neurips22_hdc_favorite}, and
prior VSA-RL studies report promising learning efficiency, with
NavHD~\cite{lee2025navhd} providing systematic multi-seed evidence in robotic
navigation; VSPG examines whether this behavior extends to discrete-action
policy gradients. VSPG further connects to kernelized and log-linear policy
search: kernel policies represent action scores as expansions over
experience~\cite{kernel5_bagnell2003policy,kernel3_pmlr-v38-lever15,
kernel4_zhang2025residual}, whereas VSPG superposes this expansion into one
fixed-size hypervector per action, requiring only one inner product per action
at inference. With a fixed encoder, VSPG is a log-linear softmax policy over
random features~\cite{rahimi2008kitchen,sutton1999policy}, with optimization
covered by established policy-gradient theory~\cite{mei2020global,
agarwal2021theory}. As summarized in \autoref{tab:positioning}, VSPG connects
VSA-based RL, log-linear policy gradients, and kernel policy search through an
exact VSPG update, a fixed-size kernel-memory
expansion, and a stability guarantee under bit-level corruption.

\subsection{Robustness to Model-State Corruption in RL}
\label{sec:rw_robust_rl}

%RL policies deployed on resource-constrained systems can be affected not only by
%environmental uncertainty, but also by corruption of the stored policy itself~\cite{robustness_rl, weather_noise, dulac2021challenges}.
% Beyond environmental uncertainty,
RL policies deployed on resource-constrained systems may be affected by corruption of the stored policy itself~\cite{robustness_rl,weather_noise,dulac2021challenges}. Low-voltage operation and approximate memory can introduce bit-level errors in model parameters, reducing the reliability of autonomous decisions. Prior work mainly addresses this problem through fault-aware training, bit-error injection, or hardware-level adaptation for neural policies~\cite{bit_flip_2,bit_flip_6_rl,bit_flip_1_rl}. VSPG instead takes a complementary representation-level approach, storing a discrete-action policy as distributed vector-symbolic action memories. We evaluate degradation under quantization and model-state bit corruption, while our theoretical analysis provides an exponential failure bound in the hypervector dimension for bipolar memories at a fixed similarity margin (\autoref{tab:positioning}).

\section{Preliminaries}
\label{sec:prelim}
\subsection{Reinforcement Learning and Policy Gradients}
\label{sec:prelim_rl}

We formulate the problem as a decentralized partially observable Markov
decision process (Dec-POMDP), which includes fully observed single-agent MDPs
as a special case. A Dec-POMDP is defined as
$\mathcal{M}
=
\bigl(
\mathcal{I},
\mathcal{S},
\{\mathcal{A}_i\}_{i\in\mathcal{I}},
P,
r,
\rho_0,
\{\mathcal{X}_i\}_{i\in\mathcal{I}},
O,
\gamma
\bigr),$
where $\mathcal{I}$ is the set of agents, $\mathcal{S}$ is the state space,
$\mathcal{A}_i$ and $\mathcal{X}_i$ are the action and observation spaces of
agent $i$, $P$ is the state transition probability, $r$ is the shared reward function,
$\rho_0$ is the initial-state distribution, $O$ are the observations,
and $\gamma\in[0,1]$ is the discount factor. At time $t$, the environment is
in state $s_t$, each agent receives observation $x_t^i$, and selects
$a_t^i \sim \pi_{\theta_i}(\cdot \mid x_t^i).$
The joint action
$\mathbf{a}_t=(a_t^i)_{i\in\mathcal{I}}$ determines the subsequent transition
and reward. 
% Single-agent RL is recovered when $|\mathcal{I}|=1$, while fully
% observed settings are recovered when the policy observation determines the
% environment state, in particular when $x_t^i=s_t$.
Let $\boldsymbol{\theta}=\{\theta_i\}_{i\in\mathcal{I}}$ denote the joint
policy parameters. The objective is to maximize the expected discounted return
\begin{equation}
\label{eq:rl_objective}
J(\boldsymbol{\theta})
=
\mathbb{E}_{\tau\sim\boldsymbol{\pi}_{\boldsymbol{\theta}}}
\left[
\sum_{t=0}^{T-1}
\gamma^t r_t
\right],
\end{equation}
where $\tau$ is a trajectory induced by the joint policy and environment.
For a factorized decentralized policy, the policy gradient for agent $i$ is
\begin{equation}
\label{eq:pg}
\nabla_{\theta_i}J(\boldsymbol{\theta})
=
\mathbb{E}_{\boldsymbol{\pi}_{\boldsymbol{\theta}}}
\left[
\sum_{t=0}^{T-1}
\hat{A}_t^i
\nabla_{\theta_i}
\log\pi_{\theta_i}(a_t^i\mid x_t^i)
\right],
\end{equation}
where $\hat{A}_t^i$ is an advantage estimate
~\cite{sutton1998reinforcement,sutton1999policy}, obtained from returns or an
actor--critic estimator such as Generalized Advantage Estimation (GAE)~\cite{actor_critic,gae,ppo,mappo}. VSPG
uses the same policy-gradient objective but represents each categorical actor
with action hypervectors, yielding an advantage-weighted bundling update that
applies to both single- and multi-agent settings.
% where $\hat{A}_t^i$ is an advantage estimate
% ~\cite{sutton1998reinforcement,sutton1999policy}. It may be computed from
% returns and an action-independent baseline or estimated by an actor--critic
% method such as generalized advantage estimation (GAE)
% ~\cite{actor_critic,gae,ppo,mappo}. Positive advantage increases the relative
% probability of the sampled action, whereas negative advantage decreases it.
% VSPG retains this objective and advantage signal but parameterizes each
% categorical actor with action hypervectors, turning its score-function update
% into advantage-weighted hypervector bundling. The same update therefore applies
% to single- and multi-agent settings, whether the advantage is estimated locally
% or by a centralized critic.

\subsection{VSA Basics}
\label{sec:hdc_basics}

VSA encodes inputs as high-dimensional hypervectors and computes through
bundling (superposition), binding, and similarity search. Although bundling and binding can be
implemented in different ways, in this paper they refer to element-wise addition and element-wise multiplication,
respectively. Let $\mathbf{h}=\varphi(x)\in\mathbb{R}^{D}$ denote the
hypervector produced by a fixed encoder $\varphi$, normalized when used for
similarity search. Random hypervectors are quasi-orthogonal, allowing many
items to be bundled with limited interference.

% VSA represents inputs as high-dimensional hypervectors and performs
% computation through simple algebraic operations: superposition, which
% accumulates hypervectors into one memory; binding, which composes
% hypervectors into an associated hypervector nearly orthogonal to its
% inputs; and similarity search.
% Let $\mathbf{h} = \varphi(x) \in \mathbb{R}^{D}$ denote the hypervector
% produced by a fixed encoder $\varphi$, where $D$ is the hypervector dimension. We assume hypervectors are normalized when used for similarity search.
% High-dimensional random hypervectors are quasi-orthogonal, allowing many to be superposed while remaining approximately distinguishable.
% A key property of high-dimensional random representations is
% quasi-orthogonality: unrelated random hypervectors have near-zero inner products, so many symbols can be superposed while remaining approximately separable.

{\begin{lemma}[Quasi-orthogonality]
\label{lemma:quasi}
Let $\mathbf{u},\mathbf{v}\in\mathbb{R}^{D}$ be independent random unit
vectors, at least one of which is uniformly distributed on the sphere.
Then $\mathbb{E}[\mathbf{u}^{\top}\mathbf{v}]=0$ and, for every
$\delta\in(0,1)$,
\[
\Pr\Bigl[\,|\mathbf{u}^{\top}\mathbf{v}|\ge\sqrt{\tfrac{2\ln(2/\delta)}{D}}\,\Bigr]\le\delta .
\]
\end{lemma}
\begin{proof}
Condition on $\mathbf{v}$. By rotational invariance,
$\mathbf{u}^{\top}\mathbf{v}$ is distributed as one coordinate of a unit
vector, whose spherical-cap measure satisfies
$\Pr[|u_1|\ge t]\le 2e^{-Dt^{2}/2}$~\cite{vershynin2018high}; the tail
bound and zero mean follow.
% Condition on $\mathbf{v}$. By rotational invariance,
% $\mathbf{u}^{\top}\mathbf{v}$ is distributed as one coordinate of a uniform
% unit vector, whose spherical-cap measure satisfies
% $\Pr[|u_1|\ge t]\le 2e^{-Dt^{2}/2}$~\cite{vershynin2018high}; the tail
% bound and the zero mean follow.
\end{proof}}
% \begin{lemma}[Quasi-orthogonality]
% \label{lemma:quasi}
% Let $\mathbf{u},\mathbf{v}\in\mathbb{R}^{D}$ be independent unit-norm random
% hypervectors drawn from an isotropic distribution. Then
% \[
%     \mathbb{E}[\mathbf{u}^{\top}\mathbf{v}] = 0,
%     \qquad
%     |\mathbf{u}^{\top}\mathbf{v}| = O(D^{-1/2})
% \]
% with high probability.
% \end{lemma}
This property enables associative memory by bundling many examples into a
single prototype. A standard VSA classifier stores one prototype per class,
\begin{equation}
    \mathbf{C}_{\ell}
    =
    \sum_{i:y_i=\ell}
    w_i \phi(x_i).
\end{equation}
Prototypes superpose class examples, optionally weighted~\cite{onlinehd}, and predict by similarity search.
% In the simplest case, $w_i=1$, so each prototype is the unweighted
% superposition of encoded examples from class $\ell$. More adaptive VSA training
% rules assign sample-dependent weights~\cite{onlinehd}.
% Given a query $x_q$, prediction is performed by similarity search:
$\hat{y}=\arg\max_{\ell}    \delta\!\left(\phi(x_q),\mathbf{C}_{\ell}\right)$, where $\delta$ is an inner-product-based similarity such as cosine similarity. Quasi-orthogonality limits interference, while repeated or similar components reinforce. Under inner-product scoring, prototypes are linear weights over fixed hypervector features and may be formed by bundling or gradient-based optimization~\cite{neurips22_hdc_favorite}. VSPG applies this view to discrete-action policies: action hypervectors serve as policy weights, encoded states as fixed features, and policy-gradient learning becomes advantage-weighted bundling, yielding compressed kernel memories over visited states.
\section{Vector-Symbolic Policy Gradient}
%\section{Method: Vector Symbolic Policy Gradient}

\begin{algorithm}[h]
\caption{Vector-Symbolic Policy Gradient (VSPG)}
\label{alg:vspg}
\begin{algorithmic}[1]
\REQUIRE Fixed normalized encoder $\boldsymbol{\phi}$, action HVs
$\mathbf{C}=\{\mathbf{c}_a\}_{a\in\mathcal{A}}$, temperature $\tau$,
learning rate $\eta$
\STATE Initialize each $\mathbf{c}_a$ and rescale to $\|\mathbf{c}_a\|_2=1$
\FOR{each batch of episodes}
    \STATE Collect transitions using $\mathbf{s}_t=\boldsymbol{\phi}(x_t)$ and
    $\pi(a\mid x_t)=\operatorname{softmax}_a\!\bigl(\tau\,\mathbf{c}_a^{\top}\mathbf{s}_t\bigr)$
    \STATE Estimate advantages $\{A_t\}$
    \STATE $\Lambda_{t,a}\leftarrow
    A_t\,\tau\bigl(\mathbf{1}[a=a_t]-\pi(a\mid x_t)\bigr)$
    \STATE $\mathbf{C}\leftarrow\mathbf{C}+\eta\,\boldsymbol{\Lambda}^{\top}\mathbf{S}$
    \STATE $\mathbf{c}_a\leftarrow\mathbf{c}_a/\|\mathbf{c}_a\|_2$ for each $a\in\mathcal{A}$
\ENDFOR
\end{algorithmic}
\end{algorithm}

\subsection{Policy Representation and Update}

VSPG turns the VSA prototype classifier of the preliminaries into a
stochastic policy: class prototypes become action hypervectors, and
similarity scores become policy logits. Throughout, $x_t$ denotes the
policy input at time $t$ --- the state in fully observed tasks and the
observation otherwise. Every encoder normalizes its raw output
$\tilde{\boldsymbol{\phi}}(x)\in\mathbb{R}^{D}$ before it reaches the actor,
\begin{equation}
  \boldsymbol{\phi}(x)
  =
  \frac{\tilde{\boldsymbol{\phi}}(x)}{\|\tilde{\boldsymbol{\phi}}(x)\|_2},
  \qquad
  \|\boldsymbol{\phi}(x)\|_2=1.
  \label{eq:enc_norm}
\end{equation}
We write $\mathbf{s}_t=\boldsymbol{\phi}(x_t)$ and stack a batch of $T$
encoded inputs as $\mathbf{S}\in\mathbb{R}^{T\times D}$. 
% The actor parameters are the action hypervectors $\mathbf{C}\in\mathbb{R}^{|\mathcal{A}|\times D}$, one unit-norm row $\mathbf{c}_a$ per action. $\mathbf{c}_a$ is initialized independently from an isotropic random distribution and normalized to unit norm; the initialization is independent of the fixed encoder. 
The actor parameters are the action hypervectors
$\mathbf{C}\in\mathbb{R}^{|\mathcal{A}|\times D}$, with one unit-norm
row $\mathbf{c}_a$ per action. Unless otherwise stated, each row is
initialized independently from an isotropic Gaussian distribution and
normalized to unit norm, independently of the fixed encoder.
The policy scores each action by an inner product and takes a softmax,
\begin{equation}
  \pi(a\mid x)
  = \operatorname{softmax}_a\!\bigl(\tau\,\mathbf{C}\,\boldsymbol{\phi}(x)\bigr),
  \label{eq:policy}
\end{equation}
so, both factors being unit-norm, each logit is a bounded scaled cosine,
\begin{equation}
  \ell_a(x)
  = \tau\,\mathbf{c}_a^{\top}\boldsymbol{\phi}(x)
  \in [-\tau,\tau],
  \label{eq:logit}
\end{equation}
and $\tau>0$ sets how sharp the policy can be; a batch is scored in the
single matrix product $\tau\,\mathbf{S}\mathbf{C}^{\top}$. 
Because both factors are unit norm, each logit lies in $[-\tau,\tau]$. Thus, $\tau$ controls both the concentration of the stochastic policy and the scale of the policy-gradient update.
% {Because both factors are unit norm, the logits never leave $[-\tau,\tau]$, so at every parameter setting the policy satisfies $\pi(a\mid x)\le\bigl(1+(|\mathcal{A}|-1)e^{-2\tau}\bigr)^{-1}$. The unit-norm constraint therefore acts as a built-in entropy floor during training, which bounds update magnitudes and preserves exploration, and the greedy readout $a^{*}=\arg\max_a\mathbf{c}_a^{\top}\boldsymbol{\phi}(x)$ removes this floor at deployment. 
% The choice of $\tau$ thus controls both the step scale and the achievable determinism of the stochastic policy.}
Given actions $\{a_t\}$ and advantages $\{A_t\}$, we form
$\boldsymbol{\Lambda}\in\mathbb{R}^{T\times|\mathcal{A}|}$ and update:
\begin{align}
\Lambda_{t,a}
&=
A_t\,\tau\!\left(
\mathbf{1}[a=a_t]-\pi(a\mid x_t)
\right),
\label{eq:lambda}
\\[3pt]
\mathbf{C}
&\leftarrow
\operatorname{row\text{-}norm}\!\left(
\mathbf{C}
+\eta\,\boldsymbol{\Lambda}^{\top}\mathbf{S}
\right).
\label{eq:update}
\end{align}
where $\eta>0$ is the learning rate,
$\mathbf{G}=\boldsymbol{\Lambda}^{\top}\mathbf{S}$ is one matrix
multiplication, and row normalization restores unit norms. The update is
closed-form: no gradient is backpropagated through the encoder or a neural
actor, and no optimizer state is kept. Since the per-row step scales with
$\eta\tau$ while $\tau$ also sets the logit scale \eqref{eq:logit}, the two
should therefore be selected jointly. 
VSPG is compatible with any advantage estimator. When Monte-Carlo estimates
are unreliable due to sparse rewards, a critic may optionally be used during
training to compute advantages, as in actor--critic methods~\cite{actor_critic}. The critic does
not update the action hypervectors or the encoder and is discarded at
deployment.

% The advantages may come from any
% estimator, consumed as fixed per-step weights; when rewards are too sparse
% for Monte-Carlo estimates, an auxiliary value function on the encoded
% states provides GAE advantages during training only --- it never updates
% $\mathbf{C}$ or the encoder and is discarded at deployment.

\subsubsection{Encoders}
Every encoder is built from a fixed random base map $\varphi$, drawn once
and never trained:
\begin{equation}
\begin{aligned}
  \varphi_{\mathrm{RFF}}(x)
  &= \sqrt{\frac{2}{D}}\,
     \cos(\mathbf{W}_{\mathrm{RFF}}x+\mathbf{b})
  && \text{(RFF)},\\
  \varphi_{\mathrm{FHRR}}(x)
  &= \sqrt{\frac{2}{D}}
     \begin{bmatrix}
       \cos(\mathbf{W}_{\mathrm{FHRR}}x)\\
       \sin(\mathbf{W}_{\mathrm{FHRR}}x)
     \end{bmatrix}
  && \text{(FHRR)},\\
  \varphi_{\mathrm{Basis}}(x)
  &= \rho\!\left(\mathbf{W}_{\mathrm{Basis}}x\right),
  \qquad
  \rho\in\{\operatorname{id},\operatorname{sign}\}
  && \text{(Basis)}.
\end{aligned}
\label{eq:base_maps}
\end{equation}
Here,
$\mathbf{W}_{\mathrm{RFF}},\mathbf{W}_{\mathrm{Basis}}
\in\mathbb{R}^{D\times d}$ and
$\mathbf{W}_{\mathrm{FHRR}}\in\mathbb{R}^{(D/2)\times d}$.
The entries of $\mathbf{W}_{\mathrm{RFF}}$ and
$\mathbf{W}_{\mathrm{FHRR}}$ are drawn independently from
$\mathcal{N}(0,\sigma^{-2})$, those of
$\mathbf{W}_{\mathrm{Basis}}$ from $\mathcal{N}(0,1)$, and
$b_i\sim\mathcal{U}(0,2\pi)$.
We define $\operatorname{id}(z)=z$.
The raw encoding $\widetilde{\boldsymbol{\phi}}(x)$ is either
$\varphi(x)$ applied directly or a superposition of bound
$\varphi$-encodings, followed by \eqref{eq:enc_norm}.
The FHRR map represents $e^{i\mathbf{W}x}$ through its real and
imaginary parts, giving unit norm and
$\varphi_{\mathrm{FHRR}}(x)^\top\varphi_{\mathrm{FHRR}}(y)=\frac{2}{D}\sum_{j=1}^{D/2}\cos\!\bigl(\mathbf{w}_j^\top(x-y)\bigr).$
After normalization, the Basis map approximates cosine similarity when
$\rho=\operatorname{id}$ and the angular kernel
$1-2\theta(x,y)/\pi$ when $\rho=\operatorname{sign}$, where
$\theta(x,y)=\arccos\!\left(\frac{x^\top y}{\|x\|_2\|y\|_2}\right).$

Normalized inner products
then concentrate around an encoder-specific similarity with
$\kappa(x,x)=1$, tightening as $D$
grows~\cite{rff,fhrr,thomas2021theoretical}: when
$\mathbb{E}\bigl[\tilde{\boldsymbol{\phi}}(x)^{\!\top}
\tilde{\boldsymbol{\phi}}(y)\bigr]=k(x,y)$, the normalization identifies
\begin{equation}
  \boldsymbol{\phi}(x)^{\!\top}\boldsymbol{\phi}(y)
  \;\approx\;
  \frac{k(x,y)}{\sqrt{k(x,x)\,k(y,y)}}
  \;=\;
  \kappa(x,y).
  \label{eq:cosine_kernel}
\end{equation}
The analysis below uses only
$\boldsymbol{\phi}(x_t)^{\!\top}\boldsymbol{\phi}(x)\approx\kappa(x_t,x)$;
the encoder determines $\kappa$, and hence how advantage evidence
generalizes across inputs, while the VSA dimensionality $D$ controls how
accurately the kernel is approximated, which the
dimensionality ablation probes directly.

{\paragraph{Computational and memory cost.}
Encoding costs one matrix--vector product, $O(Dd)$ for input dimension $d$, scoring costs $O(|\mathcal{A}|D)$, and one update costs $O(T|\mathcal{A}|D)$ for $\boldsymbol{\Lambda}^{\top}\mathbf{S}$ plus $O(T|\mathcal{A}|)$ for the softmax terms, with no optimizer state. Because $\mathbf{W}$ and $\mathbf{b}$ are random and never trained, deployment stores the $|\mathcal{A}|\times D$ action memories together with either the projection or the seed that regenerates it, and the bipolar variant stores one bit per memory coordinate.}
% ─────────────────────────────────────────────────────────────────────────────
\subsection{Theoretical Analysis of VSPG}
%\subsection{Projected PG and Kernel Memories}
\label{sec:theory}
% The actor raises two questions: what is the policy-gradient update for the
% policy \eqref{eq:policy}, and what do the trained action hypervectors
% store? Proposition \autoref{prop:equiv} answers the first --- the
% closed-form rule \eqref{eq:update} \emph{is} the softmax policy-gradient
% step, written as advantage-weighted bundling followed by row-wise sphere
% projection --- and Proposition \autoref{prop:kernel} the second: each
% trained hypervector is a compressed kernel expansion over experience.
% Throughout, let $\{(x_t, a_t, A_t)\}_{t=1}^{T}$ be a batch of transitions
% and define the empirical surrogate
% \begin{equation}
%   \hat{J}(\mathbf{C}) \;=\; \sum_{t=1}^{T} A_t \log \pi_{\mathbf{C}}(a_t \mid x_t),
%   \label{eq:surrogate}
% \end{equation}
% whose gradient is the sampled actor update induced by the chosen advantage
% estimates --- the REINFORCE estimator of \eqref{eq:pg} for unbiased
% return-minus-baseline
% advantages~\cite{sutton1998reinforcement}, and the corresponding sampled
% step for GAE, group-normalized, or clipped estimators. We first state the
% projection fact.
VSPG's actor design raises two questions: what is the policy-gradient update for the
policy \eqref{eq:policy}, and what do the trained action hypervectors
store? Proposition \autoref{prop:equiv} answers the first --- the
closed-form rule \eqref{eq:update} \emph{is} the softmax policy-gradient
step, written as advantage-weighted bundling followed by row-wise sphere
projection --- and Proposition \autoref{prop:kernel} the second: each
trained hypervector is a compressed kernel expansion over experience.
Throughout, let $\{(x_t, a_t, A_t)\}_{t=1}^{T}$ be a batch of transitions
and define the empirical surrogate
\begin{equation}
  \hat{J}(\mathbf{C})
  \;=\;
  \sum_{t=1}^{T}
  A_t \log \pi_{\mathbf{C}}(a_t \mid x_t),
  \label{eq:surrogate}
\end{equation}
where $\{A_t\}$ are treated as fixed weights when differentiating with
respect to $\mathbf{C}$. We first state the projection fact.
\begin{lemma}[Closest unit vector]
\label{lem:proj}
For any $\mathbf{v}\neq\mathbf{0}$, the unique unit vector closest to
$\mathbf{v}$ in Euclidean distance is $\mathbf{v}/\|\mathbf{v}\|_2$.
\end{lemma}
\begin{proof}
On $\|\mathbf{u}\|_2=1$,
$\|\mathbf{u}-\mathbf{v}\|_2^2 = 1 + \|\mathbf{v}\|_2^2 - 2\,\mathbf{u}^\top\mathbf{v}$,
so minimizing distance maximizes $\mathbf{u}^\top\mathbf{v}$, which by
Cauchy--Schwarz occurs uniquely at $\mathbf{u}=\mathbf{v}/\|\mathbf{v}\|_2$.
\end{proof}
\begin{proposition}[The VSPG update is a projected policy-gradient step]
%\begin{proposition}[The policy gradient is bundling]
\label{prop:equiv}
Fix the encoder $\boldsymbol{\phi}$ and the policy \eqref{eq:policy}. For any
advantages $\{A_t\}$,
\begin{equation*}
  \nabla_{\mathbf{C}}\hat{J}(\mathbf{C})
  \;=\; \boldsymbol{\Lambda}^{\top}\mathbf{S} \;=\; \mathbf{G},
\end{equation*}
with $\boldsymbol{\Lambda}$ as in \eqref{eq:lambda}: the bundling term in
\eqref{eq:update} is exactly the sampled policy gradient, and
\eqref{eq:update} is a gradient-ascent step on $\hat{J}$ followed by
row-wise projection onto the unit sphere (Lemma~\ref{lem:proj}).
\end{proposition}
\begin{proof}
Write $\log\pi(a_t\mid x_t)=\ell_{a_t}(x_t)-\log Z_t$ with
$Z_t=\sum_b e^{\ell_b(x_t)}$. A logit $\ell_b$ depends on row
$\mathbf{c}_a$ only when $b=a$, with
$\nabla_{\mathbf{c}_a}\ell_a=\tau\,\mathbf{s}_t$, and the log-sum-exp
derivative gives
$\nabla_{\mathbf{c}_a}\log Z_t=\pi(a\mid x_t)\,\tau\,\mathbf{s}_t$.
Subtracting yields the softmax score
\begin{equation}
  \nabla_{\mathbf{c}_a}\log\pi(a_t\mid x_t)
  = \tau\!\left(\mathbf{1}[a=a_t]-\pi(a\mid x_t)\right)\mathbf{s}_t:
  \label{eq:score}
\end{equation}
the encoded state is added to the taken action's row and subtracted from
every row in proportion to its current probability. Weighting each score
by $A_t$ and summing over the batch matches $\boldsymbol{\Lambda}$
entrywise, so
$\nabla_{\mathbf{C}}\hat{J}=\boldsymbol{\Lambda}^{\top}\mathbf{S}=\mathbf{G}$.
Row normalization rescales each row to unit length, which by
Lemma~\ref{lem:proj} projects $\mathbf{C}+\eta\mathbf{G}$ onto the
product of unit spheres
$\mathcal{M}=\{\mathbf{C} : \|\mathbf{c}_a\|_2 = 1 \;\forall a\}$.
\end{proof}

{Proposition~\autoref{prop:equiv} gives the exact projected policy-gradient update.
The following relates this projection to the intrinsic geometry of the unit sphere.

\begin{corollary}[First-order form of the normalized step]
\label{cor:riemann}
Let $\mathbf{g}_a=\nabla_{\mathbf{c}_a}\hat J(\mathbf{C})$, and let
$\mathbf{c}_a^+$ denote row $a$ after \eqref{eq:update}. If
$\eta\|\mathbf{g}_a\|_2\le 1/4$, then
\begin{equation}
\left\|
\mathbf{c}_a^+-\mathbf{c}_a
-\eta(\mathbf{I}-\mathbf{c}_a\mathbf{c}_a^\top)\mathbf{g}_a
\right\|_2
\le
3\eta^2\|\mathbf{g}_a\|_2^2.
\end{equation}
Thus, the normalized update is first-order equivalent to Riemannian gradient
ascent on the product of unit spheres. Moreover,
$\|\mathbf{g}_a\|_2\le\tau\sum_t|A_t|$, so the condition holds whenever
$\eta\tau\sum_t|A_t|\le1/4$.
\end{corollary}

\begin{proof}
Write
$\mathbf{g}_a=s\mathbf{c}_a+\mathbf{h}$, where
$s=\mathbf{c}_a^\top\mathbf{g}_a$ and
$\mathbf{h}=(\mathbf{I}-\mathbf{c}_a\mathbf{c}_a^\top)\mathbf{g}_a
\perp\mathbf{c}_a$.
With
$N=\|\mathbf{c}_a+\eta\mathbf{g}_a\|_2$,
$
\mathbf{c}_a^+
=
\frac{(1+\eta s)\mathbf{c}_a+\eta\mathbf{h}}{N}.$
Let $u=\eta\|\mathbf{g}_a\|_2\le1/4$. Then $N\ge1-u\ge3/4$ and
$
\left|\frac{1+\eta s}{N}-1\right|\le u^2,
\qquad
\left|\frac1N-1\right|\le\frac{u}{1-u}\le2u.
$
Since the radial and tangential terms are orthogonal,
$
\left\|
\mathbf{c}_a^+-\mathbf{c}_a-\eta\mathbf{h}
\right\|_2
\le
\sqrt{u^4+4u^4}
<3u^2.
$
Finally,
$\|\mathbf{g}_a\|_2
\le\sum_t|\Lambda_{t,a}|\|\mathbf{s}_t\|_2
\le\tau\sum_t|A_t|$.
\end{proof}}

\begin{figure*}
    \centering
    \includegraphics[width=0.98\linewidth]{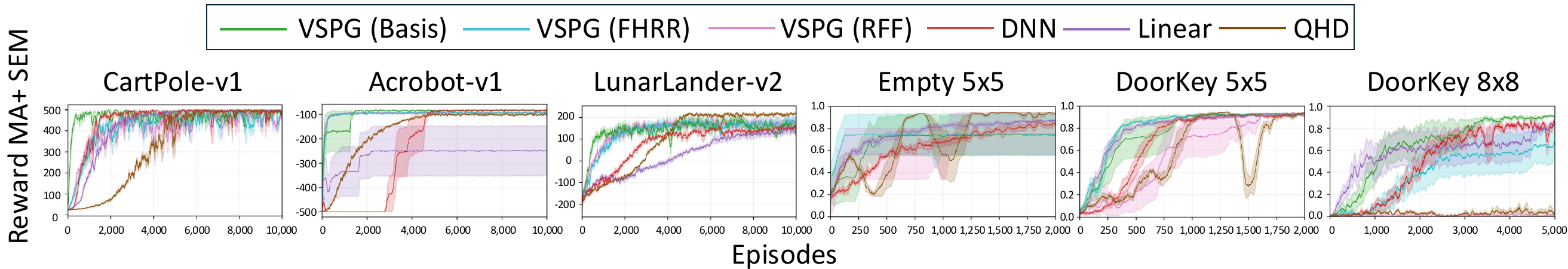}
    \caption{Overall performance analysis of VSPG compared with different baselines and encoders runs on 5 seeds. Except for \texttt{LunaLander-v2}, VSPG-based achieves the fastest convergence and competitive performance compared to baselines.}
    \label{fig:minigrid_control}
\end{figure*}
{Proposition~\ref{prop:equiv} places VSPG within standard policy-gradient theory: the actor is a log-linear policy over fixed hypervector features~\cite{sutton1999policy,mei2020global}, and under the unit-norm constraint its exact gradient step takes the form of an VSA bundling operation. Alternative advantage estimates or surrogate weights change only the entries of $\boldsymbol{\Lambda}$.}
%Informally, Proposition~\ref{prop:equiv} shows that the closed-form rule is exactly the policy gradient of a linear softmax actor over fixed VSA features, expressed as advantage-weighted bundling followed by projection, rather than a new optimizer. Alternative advantage estimates or surrogate weights change only the entries of $\boldsymbol{\Lambda}$. We use the closed form because it requires neither backpropagation nor optimizer state and exposes the structure of the learned action memories: repeated updates leave each action hypervector as a superposition of its initialization and encoded training inputs, whose inner product with a query yields a kernel expansion. Proposition~\ref{prop:kernel} formalizes this result.

\begin{proposition}[Exact expansion of trained action hypervectors]
\label{prop:kernel}
Run Algorithm \autoref{alg:vspg} for any number of updates from unit-norm
initialization $\{\mathbf{c}_a^{(0)}\}$, and let
$\{x_k\}_{k=1}^{N}$ collect all inputs visited during training. Then there
exist scalars $\beta_a > 0$ and $\{\alpha_{k,a}\}$ such that each trained
action hypervector is
\begin{equation}
  \mathbf{c}_a
  \;=\;
  \beta_a\,\mathbf{c}_a^{(0)}
  \;+\;
  \sum_{k=1}^{N} \alpha_{k,a}\,\boldsymbol{\phi}(x_k),
  \label{eq:span}
\end{equation}
and consequently the logit at any query input $x$ decomposes as
\begin{equation}
\begin{aligned}
  \ell_a(x)
=
  &\underbrace{\tau\beta_a\,\mathbf{c}_a^{(0)\top}\boldsymbol{\phi}(x)}_{
    \text{initialization; } O(1/\sqrt{D})}+
  \tau\!\sum_{k=1}^{N} \alpha_{k,a}\,
  \underbrace{\boldsymbol{\phi}(x_k)^{\!\top}\boldsymbol{\phi}(x)}_{
    \approx\, \kappa(x_k,\,x)}.
\end{aligned}
\label{eq:kernel_logit}
\end{equation}
%Each $\alpha_{k,a}$ carries the sign of the accumulated advantage-weighted
%softmax scores of input $x_k$ for action $a$, rescaled by the positive
%normalization factors.

{Writing $z_a^{(j)}>0$ for the Euclidean norm of row $a$ after the $j$-th
bundling step and before its renormalization, and $\Lambda^{(j)}$,
$\{x_t^{(j)}\}$ for the weights and inputs of the $j$-th batch, the
coefficients are
\begin{equation}
\beta_a=\prod_{j=1}^{J}\bigl(z_a^{(j)}\bigr)^{-1},
\alpha_{k,a}
=\eta\!\!\sum_{(j,t)\,:\,x_t^{(j)}=x_k}\!\!\Lambda_{t,a}^{(j)}
\prod_{i=j}^{J}\bigl(z_a^{(i)}\bigr)^{-1}.
\label{eq:coeffs}
\end{equation}
Each $\alpha_{k,a}$ is therefore a positively weighted sum of the
advantage-weighted softmax scores collected at the visits of $x_k$; in
particular, if $x_k$ is visited once at step $t$, then
$\operatorname{sign}(\alpha_{k,a})
=\operatorname{sign}\bigl(A_t(\mathbf{1}[a=a_t]-\pi(a\mid x_t))\bigr)$.}
\end{proposition}
\begin{proof}
By induction. At initialization \eqref{eq:span} holds with $\beta_a=1$,
$\alpha_{k,a}=0$. Each update \eqref{eq:update} adds
$\eta\sum_t \Lambda_{t,a}\,\mathbf{s}_t$, which lies in the span of
encoded visited inputs, and row normalization rescales the row by a
positive scalar, preserving the form of \eqref{eq:span} and the
coefficient signs. 
{Unrolling the recursion
$\mathbf{c}_a^{(j)}=\bigl(\mathbf{c}_a^{(j-1)}+\eta\sum_t\Lambda_{t,a}^{(j)}\mathbf{s}_t^{(j)}\bigr)/z_a^{(j)}$
yields the coefficients in \eqref{eq:coeffs}, whose weights are positive
because every $z_a^{(i)}$ is positive. Inner products of \eqref{eq:span}
with $\tau\boldsymbol{\phi}(x)$ give \eqref{eq:kernel_logit}.}
%
%Inner products of \eqref{eq:span} with
% $\tau\boldsymbol{\phi}(x)$ give \eqref{eq:kernel_logit}. For the brackets:
% $\mathbf{c}_a^{(0)\top}\boldsymbol{\phi}(x)$ is the cosine of two
% independent high-dimensional directions and concentrates at $O(1/\sqrt{D})$
% (\autoref{lemma:quasi}); and
% $\boldsymbol{\phi}(x_k)^{\!\top}\boldsymbol{\phi}(x)$ concentrates around
% $\kappa(x_k,x)$ --- uniformly over compact domains at rate
% $O\bigl(\sqrt{\log(1/\delta)/D}\bigr)$ for random-feature
% encoders~\cite{rff,thomas2021theoretical}, and pointwise at $O(1/\sqrt{D})$
% for projection-based and symbolic encoders by coordinate-wise
% concentration.
\end{proof}

% {\color{red}\begin{corollary}[Deployment error of the kernel view]
% \label{cor:kernel_err}
% Fix a run of Algorithm~\ref{alg:vspg} with resulting coefficients
% $(\beta_a,\{\alpha_{k,a}\})$ from \eqref{eq:span}, and let
% $\mathcal{X}$ be a compact input domain. Define
% $\varepsilon_D=\sup_{x,x'\in\mathcal{X}}
% |\boldsymbol{\phi}(x)^{\top}\boldsymbol{\phi}(x')-\kappa(x,x')|$
% and
% $\varepsilon'_D=\sup_{x\in\mathcal{X}}
% |\mathbf{c}_a^{(0)\top}\boldsymbol{\phi}(x)|$.
% Then for every query $x\in\mathcal{X}$,
% $
% \Bigl|\ell_a(x)-\tau\sum_{k=1}^{N}\alpha_{k,a}\,\kappa(x_k,x)\Bigr|
% \;\le\;
% \tau\,\beta_a\,\varepsilon'_D
% +\tau\,\|\boldsymbol{\alpha}_{\cdot,a}\|_1\,\varepsilon_D ,
% $
% by the triangle inequality applied to \eqref{eq:kernel_logit}. For
% random-feature encoders,
% $\varepsilon_D=O\bigl(\sqrt{\log(1/\delta)/D}\bigr)$ with probability at
% least $1-\delta$ over the encoder draw~\cite{rff,thomas2021theoretical};
% for projection-based and symbolic encoders the same quantity concentrates
% pointwise at rate $O(1/\sqrt{D})$ by coordinate-wise concentration; and
% $\varepsilon'_D$ obeys the random-feature rate for random unit-norm
% initialization by Lemma~\ref{lemma:quasi} combined with a union bound over
% a cover of $\mathcal{X}$.
% \end{corollary}}

Informally, up to a vanishing initialization bias, the deployed VSPG
policy is a softmax over \emph{advantage-weighted kernel scores against
experience}:
\begin{equation}
  \pi(a \mid x)
  \;\approx\;
  \operatorname{softmax}_a\!\Bigl(
    \tau \sum_{k=1}^{N} \alpha_{k,a}\, \kappa(x_k, x)
  \Bigr).
  \label{eq:kernel_policy}
\end{equation}

An advantageous transition contributes positive mass to action $a$ at
encoder-similar inputs, while evidence favoring competing actions contributes
negative mass. This kernel sharing allows VSPG to reuse each transition across
a neighborhood of observations, providing a representation-level mechanism
for sample-efficient learning. Generalization is therefore governed by the
encoder-induced similarity, while the $N$-term expansion remains superposed
in $D$ fixed coordinates and is never enumerated at inference.
% \textcolor{red}{
% This kernel sharing provides a mechanism for sample-efficient learning:
% each advantageous transition can affect the policy over a neighborhood of
% similar observations rather than only at the visited input. Thus,
% Proposition~\ref{prop:kernel} explains how VSPG can reuse experience across
% states, while the resulting learning-speed advantage is evaluated empirically
% in \autoref{fig:minigrid_control}.
% }
% Inputs that earned positive advantage under action $a$ contribute positive
% mass to $\ell_a$ at all similar inputs in proportion to $\kappa(x_k,x)$;
% inputs where competing actions were reinforced contribute negative mass.
% Generalization is thus governed by the encoder's similarity, while the
% $N$-term expansion stays superposed in $D$ fixed coordinates and is never
% enumerated at inference.

\begin{remark}[Memory readout as the sharp-kernel limit]
\label{rem:memory}
At a revisited input $x=x_m$, \eqref{eq:kernel_logit} splits into the self
term $\alpha_{m,a}$ (by unit norm) and cross terms
$\sum_{k\neq m}\alpha_{k,a}\,\kappa(x_k,x_m)$. When dissimilar inputs are
nearly orthogonal under $\boldsymbol{\phi}$, the cross terms vanish and
$\ell_a(x_m)\approx\tau\,\alpha_{m,a}$: the associative-memory retrieval
underlying VSA classification~\cite{kanerva2009hyperdimensional,onlinehd}. 
% In general, the cross terms are not noise but kernel evaluations that transfer advantage evidence to similar inputs; the encoder's similarity sets the trade-off between retrieval fidelity and generalization, and larger $D$ tightens the kernel approximation.
In general, the cross terms are not noise but kernel evaluations that transfer advantage evidence to similar inputs; the encoder’s similarity determines how this evidence is shared across inputs, and larger $D$ tightens the kernel approximation.
\end{remark}

The same update applies independently to each agent in multi-agent training, using agent-specific advantages that may be estimated by a centralized critic during training. At deployment, only the fixed encoder and action memories are retained, and actions are selected by $a^{*}=\arg\max_a\mathbf{c}_a^{\top}\boldsymbol{\phi}(x).$
By \eqref{eq:kernel_policy}, this fixed-size readout evaluates the
compressed kernel memory without storing or enumerating the training
samples.

% Finally, VSPG is an actor parameterization, not a full RL stack: the same
% update applies unchanged per agent in multi-agent training with
% agent-specific advantages (e.g., MAPPO-style GAE from a centralized critic
% used only during training), and deployment is always a fixed encoder
% followed by one action-memory readout
% $a^{*}=\arg\max_a \mathbf{c}_a^{\top}\boldsymbol{\phi}(x)$ --- by
% \eqref{eq:kernel_policy}, a constant-depth kernel lookup with no learned
% layers between encoding and logits.

{
\subsection{Robustness of the Stored Policy}
\label{sec:robust_theory}

For bipolar action memories, independent sign flips admit a direct
stability guarantee for the greedy readout.

\begin{proposition}[Bit-flip stability of the greedy readout]
\label{prop:bitflip}
Let the deployed actor store bipolar action memories
$\mathbf{c}_a\in\{-1/\sqrt{D},+1/\sqrt{D}\}^{D}$ and select actions by
$a^{*}(x)=\arg\max_{a}\mathbf{c}_a^{\top}\boldsymbol{\phi}(x)$ with
$\|\boldsymbol{\phi}(x)\|_2=1$. Suppose each stored coordinate flips sign
independently with probability $p<1/2$, giving corrupted memories
$\tilde{\mathbf{c}}_a$. Fix an input $x$ and let
$\Delta(x)=\mathbf{c}_{a^{*}}^{\top}\boldsymbol{\phi}(x)-\max_{b\neq a^{*}}\mathbf{c}_b^{\top}\boldsymbol{\phi}(x)>0$
denote the similarity margin. Then
% \begin{equation}
% \begin{aligned}
% \Pr\Bigl[\arg\max_{a}\tilde{\mathbf{c}}_a^{\top}\boldsymbol{\phi}(x)\neq a^{*}(x)\Bigr]
% \;\le\;
% 2\,|\mathcal{A}|\,\exp\!\Bigl(-\tfrac{D\,(1-2p)^{2}\,\Delta(x)^{2}}{8}\Bigr).
% \end{aligned}    
% \end{equation}
{\small
\begin{equation}
\Pr\Bigl[
\arg\max_{a}\tilde{\mathbf{c}}_a^{\top}\boldsymbol{\phi}(x)
\neq a^{*}(x)
\Bigr]
\le
2\,|\mathcal{A}|\,
\exp\!\Bigl(
-\tfrac{D(1-2p)^2\Delta(x)^2}{8}
\Bigr).
\end{equation}
}
\end{proposition}
\begin{proof}
Write $\tilde{c}_{a,j}=\sigma_{a,j}\,c_{a,j}$ with independent
$\sigma_{a,j}\in\{-1,+1\}$ and $\Pr[\sigma_{a,j}=-1]=p$, so
$\mathbb{E}[\sigma_{a,j}]=1-2p$. Then
$\tilde{\mathbf{c}}_a^{\top}\boldsymbol{\phi}(x)=(1-2p)\,\mathbf{c}_a^{\top}\boldsymbol{\phi}(x)+\zeta_a$
with $\zeta_a=\sum_{j}\bigl(\sigma_{a,j}-(1-2p)\bigr)c_{a,j}\phi_j(x)$, a sum
of independent zero-mean terms whose $j$-th term has range
$2|c_{a,j}\phi_j(x)|$. Since
$\sum_j\bigl(2|c_{a,j}\phi_j(x)|\bigr)^{2}=\tfrac{4}{D}\sum_j\phi_j(x)^{2}=\tfrac{4}{D}$,
Hoeffding's inequality gives $\Pr[|\zeta_a|\ge t]\le 2\exp(-Dt^{2}/2)$ for
every $a$. Scaling all similarities by $1-2p>0$ preserves the maximizer, so
the corrupted argmax can change only if $|\zeta_a|\ge(1-2p)\Delta(x)/2$ for
some $a$. A union bound over the $|\mathcal{A}|$ actions completes the
proof.
\end{proof}

\begin{remark}[Beyond bipolar memories]
\label{rem:bitflip}
For bipolar memories, random sign flips primarily scale the clean logits by
$(1-2p)$, with a residual perturbation that concentrates as $D$ grows.
Multi-bit quantized memories do not follow the same sign-flip model, but
independent bounded coordinate errors are likewise averaged by
low-coherence hypervectors. We therefore evaluate their robustness
empirically rather than claim the same bound.
\end{remark}

% \begin{remark}[Effect on the stochastic policy]
% \label{rem:bitflip}
% With probability at least
% $1-2|\mathcal{A}|e^{-Dt^2/2}$, all corrupted logits equal the clean logits
% scaled by $(1-2p)$ up to an additive perturbation of at most $\tau t$.
% Thus, random sign flips primarily reduce the effective temperature, with
% residual error decreasing as
% $O\!\left(\tau\sqrt{\log(|\mathcal{A}|)/D}\right)$.
% \end{remark}
}

\section{Experiments}

\label{sec:experiments}

% ============================================================
% Insert in the Experiments section
% ============================================================

\subsection{Experimental Setup and Implementation Details}

\label{sec:experimental_setup}
To evaluate VSPG and the empirical implications of our theory, we use
MiniGrid~\cite{minigrid}, classic control, and SustainGym building
control~\cite{yeh2023sustaingym}. The first two test sample efficiency and
final performance, while SustainGym tests continuous, noisy, delayed, and
multi-agent control. Tuning budgets are comparable and otherwise favor the
baselines; full details are in the supplementary material. Experiments use
either one RTX 4090 or CPUs, and all results are reproducible on CPUs alone. Unless otherwise noted, all methods use discount factor $\gamma = 0.99$, and VSPG uses dimensionality $D = 10{,}000$ for MiniGrid and classic control and $D = 5{,}000$ for SustainGym.

% \subsection{Experimental Setup and Implementation Details}
% \label{sec:experimental_setup}
% We evaluate VSPG on three groups of discrete-action reinforcement-learning
% tasks: MiniGrid~\cite{minigrid}, classic control, and SustainGym~\cite{yeh2023sustaingym} building control. MiniGrid and classic
% control evaluate sample efficiency and final task performance on standard
% single-agent benchmarks. Building control serves a different role: it tests
% VSPG in a more realistic physical-control setting with continuous observations,
% heterogeneous feature scales, exogenous weather variables, delayed dynamics,
% multi-agent coordination, and sensor-like observation noise.

\subsection{MiniGrid and Classic Control}
\label{sec:exp_minigrid_classic}

We first evaluate VSPG on standard single-agent discrete-action benchmarks:
classic-control environments from Gymnasium~\cite{gymnasium} and navigation
tasks from MiniGrid~\cite{minigrid}. These experiments assess sample
efficiency, final performance, and comparison with policy-gradient and prior
VSA-based RL baselines.
Following QHD~\cite{qhd}, the classic-control suite includes
\texttt{CartPole-v1}, \texttt{LunarLander-v2}, and
\texttt{Acrobot-v1}, providing a direct comparison with prior VSA
reinforcement learning on low-dimensional continuous-observation tasks.
We additionally evaluate \texttt{Empty-5x5}, \texttt{DoorKey-5x5}, and
\texttt{DoorKey-8x8}. These MiniGrid tasks introduce partial observability,
sparse rewards, and longer-horizon credit assignment: \texttt{Empty-5x5}
primarily tests rapid goal reaching, whereas the \texttt{DoorKey} tasks
require coordinated key pickup, door unlocking, and navigation to the goal.
Exact encoder constructions, observation preprocessing, and task-specific
configurations are provided in the supplementary material.

\begin{figure*}[t]
\centering

\begin{minipage}[t]{0.49\linewidth}
    \centering
    \includegraphics[width=\linewidth]{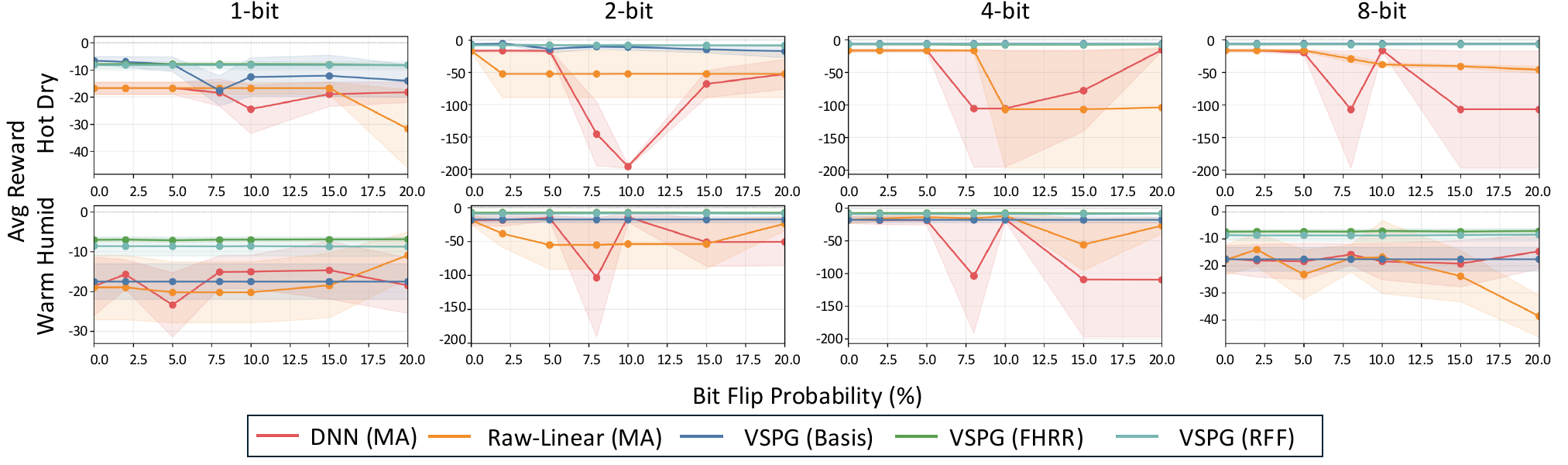}
    \caption{Bit-flip robustness on SustainGym after quantization. VSPG, DNN, andRaw-Linear actors are evaluated.}
    % \caption{Bit-flip robustness on SustainGym under post-training quantization. VSPG variants are compared with DNN and Raw-Linear actors across both climates and training pipelines; only stored actor parameters are corrupted.}
    \label{fig:bitflip}
\end{minipage}
\hfill
\begin{minipage}[t]{0.49\linewidth}
    \centering
    \includegraphics[width=\linewidth]{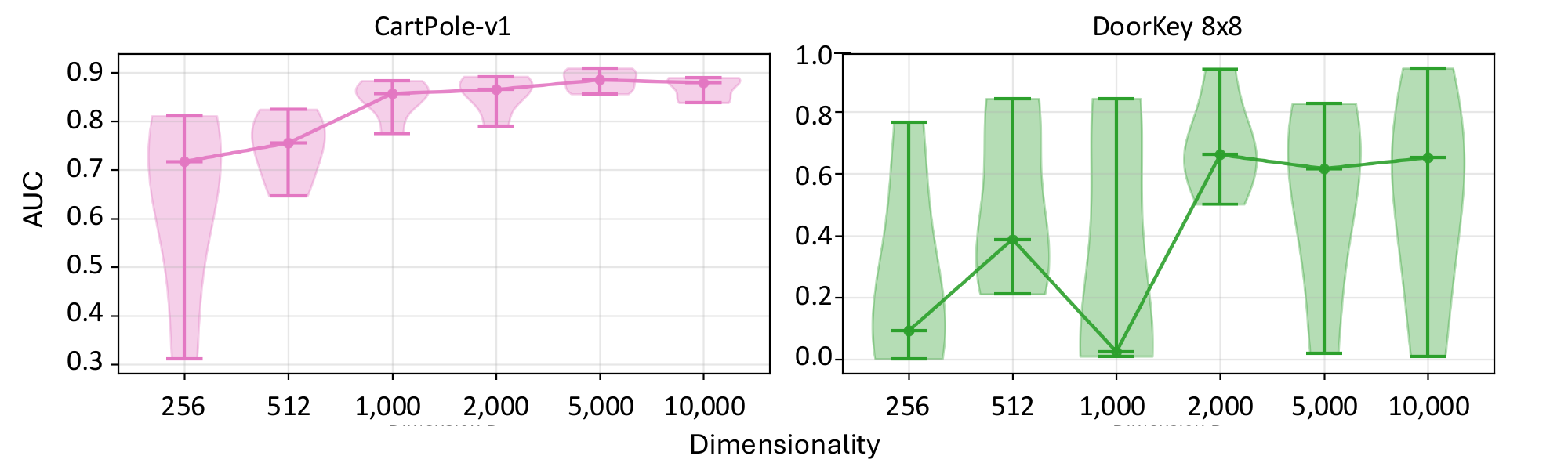}
    \caption{Dimensionality analysis on \texttt{CartPole-v1} and \texttt{DoorKey 8x8}. }
    \label{fig:dim}
\end{minipage}

\end{figure*}

\begin{figure*}[th]
    \centering
    \includegraphics[width=\linewidth]{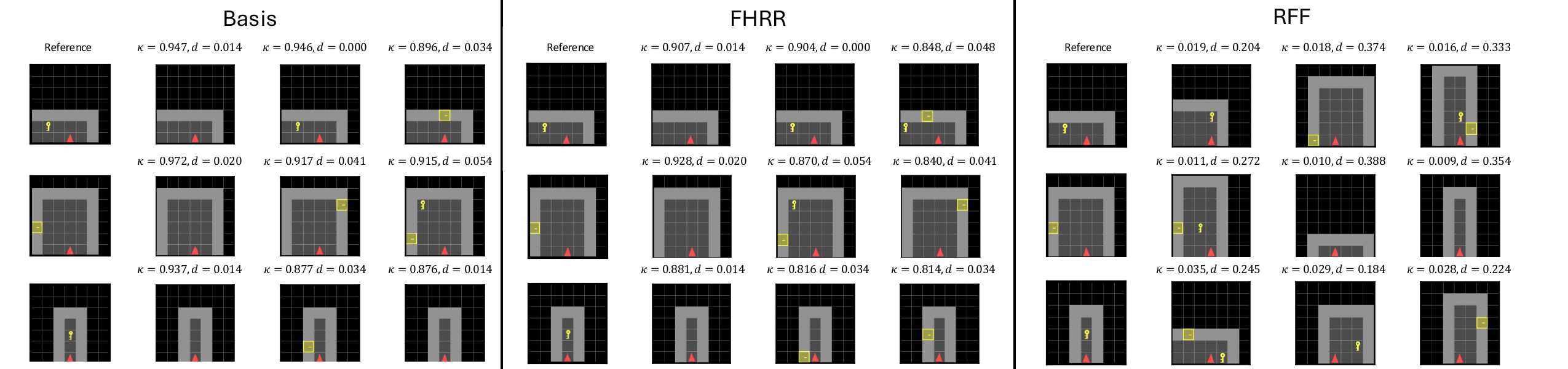}
    \caption{Encoder-induced kernel neighborhoods from trained VSPG trajectories on \texttt{DoorKey-$8\times8$}.
    % , annotated with kernel similarity $\kappa$ and normalized observation distance $d$.
    }
    % \caption{
    % Representative encoder-induced kernel neighborhoods among observations collected by trained VSPG policies on DoorKey-$8\times8$. Each retrieved observation is annotated with its kernel similarity $\kappa$ and normalized observation distance $d$ from the reference. 
    % % Bipolar and FHRR retrieve coherent semantic neighbors, whereas RFF does not.
    % }
    \label{fig:kernel_qualitative}
\end{figure*}

We compare VSPG against three baselines. \textbf{Raw-Linear} tests whether the
gains come from VSA encoding rather than from the linear action-hypervector
policy alone. The \textbf{DNN} baseline is a two-hidden-layer neural policy
trained with backpropagation under the same advantage-estimation protocol where
applicable, allowing us to compare sample efficiency against a standard neural
policy. We also include \textbf{QHD} as the closest prior VSA baseline for
single-agent discrete-action RL. Since QHD is value-based and off-policy,
whereas VSPG is policy-based and on-policy, we match the total episode budget
rather than the number of parameter updates. MiniGrid experiments use
$2{,}000$ episodes for $5\times5$ grids and $5{,}000$ episodes for the
$8\times8$ grid, while classic-control experiments use $10{,}000$ episodes.
All VSPG and policy-gradient baselines are trained with REINFORCE except for
\texttt{Acrobot-v1}, where all policy-gradient methods use GAE with the same
PPO-style clipped importance weighting~\cite{ppo}, since sparse rewards often
keep vanilla REINFORCE near the minimum return of $-500$. 
We report the mean and standard error over five seeds, using the mean reward over the final 100 episodes of each run as the main summary metric.

\noindent\textbf{Results.}
\autoref{fig:minigrid_control} shows that VSPG learns substantially faster than
the DNN and linear actors on classic control, especially on
\texttt{CartPole-v1} and \texttt{Acrobot-v1}. On \texttt{LunarLander-v2},
VSPG improves faster early in training, while all methods reach broadly
comparable final performance. VSPG also remains competitive across MiniGrid,
including the more difficult and variable \texttt{DoorKey-8x8}. QHD is
competitive on simpler classic-control tasks, but deteriorates as task
complexity and partial observability increase: it becomes unstable and
collapses on \texttt{DoorKey-5x5}, and learns almost nothing on
\texttt{DoorKey-8x8}. Overall, these results are consistent with the
kernel-sharing mechanism in Proposition~\ref{prop:kernel}, through which
advantage evidence is reused across encoder-similar observations, supporting
sample-efficient learning in both fully and partially observable tasks without
a neural actor.

% \noindent\textbf{Results.}
% \autoref{fig:minigrid_control} summarizes the learning curves across classic
% control and MiniGrid environments. On \texttt{CartPole-v1} and
% \texttt{Acrobot-v1}, VSPG variants converge substantially faster than the DNN and
% linear baselines, showing that the proposed vector-symbolic policy update can
% learn effective discrete-action policies without a neural actor. On
% \texttt{LunarLander-v2}, the DNN baseline eventually reaches comparable or
% slightly stronger performance, but VSPG still improves rapidly in the early
% stage of training. In MiniGrid, VSPG also exhibits competitive performance on
% \texttt{Empty-5x5}, \texttt{DoorKey-5x5}, and \texttt{DoorKey-8x8}, although the
% larger \texttt{DoorKey-8x8} task remains more challenging and shows higher
% variance across seeds. Overall, these results indicate that VSPG is particularly
% effective in terms of sample efficiency, while retaining competitive final
% performance across both low-dimensional control and symbolic navigation tasks.
\subsection{Building Control and Bit-flip Robustness}
\begin{table}[t]
\centering
\caption{Average reward per step (mean $\pm$ std) on SustainGym building
control after 500 training episodes.}
\label{tab:sustaingym_results}

\setlength{\tabcolsep}{3pt}
\renewcommand{\arraystretch}{0.95}
\footnotesize

\resizebox{\linewidth}{!}{%
\begin{tabular}{lcc}
\toprule
Method & Hot-Dry & Warm-Humid \\
\midrule
DNN (Multi-Agent)        & $-13.41 \pm 0.26$ & $-12.12 \pm 2.79$ \\
% DNN (Single)             & $-17.81 \pm 5.12$ & $-19.58 \pm 9.42$ \\
Raw-Linear (Multi-Agent) & $-13.02 \pm 0.22$ & $-10.92 \pm 0.13$ \\
% Raw-Linear (Single)      & $-16.57 \pm 7.09$ & $-11.95 \pm 2.05$ \\
\midrule
VSPG (Basis)           & $-40.36 \pm 43.21$ & $-79.65 \pm 80.99$ \\
VSPG (FHRR)              & $-7.89 \pm 1.41$ & $\mathbf{-7.11 \pm 0.81}$ \\
VSPG (RFF)               & $\mathbf{-7.28 \pm 0.64}$ & $-7.94 \pm 0.12$ \\
\bottomrule
\end{tabular}%
}
\end{table}
% \begin{table}[t]
% \caption{Average reward per step (mean $\pm$ std) on SustainGym building control after 500 training episodes.}
% \centering
% \begin{minipage}{0.98\linewidth}
% \centering
% \vspace{-2mm}
% \label{tab:sustaingym_results}
% \small
% \begin{tabular}{lcc}
% \toprule
% Method & Hot-Dry & Warm-Humid \\
% \midrule
% DNN (Multi-Agent)         & $-13.41 \pm 0.26$ & $-12.12 \pm 2.79$ \\
% DNN (Single)        & $-17.81 \pm 5.12$ & $-19.58 \pm 9.42$ \\
% Raw-Linear (Multi-Agent)  & $-13.02 \pm 0.22$ & $-10.92 \pm 0.13$ \\
% Raw-Linear (Single) & $-16.57 \pm 7.09$ & $-11.95 \pm 2.05$ \\
% \midrule
% VSPG (Bipolar)      & $-40.36 \pm 43.21$ & $-79.65 \pm 80.99$ \\
% VSPG (FHRR)         & $-7.89 \pm 1.41$ & $\mathbf{-7.11 \pm 0.81}$ \\
% VSPG (RFF)          & $\mathbf{-7.28 \pm 0.64}$ & $-7.94 \pm 0.12$ \\
% \bottomrule
% \end{tabular}
% \vspace{-5mm}
% \end{minipage}
% \end{table}

We evaluate VSPG on SustainGym~\cite{yeh2023sustaingym}, a
multi-agent building-control benchmark with noisy, non-stationary physical observations. We compare against DNN and Raw-Linear actors trained under the same decentralized-actor, centralized-critic multi-agent pipeline. The centralized critic is used only during training to compute GAE advantages and
is discarded at deployment. QHD is omitted because it is a single-agent, value-based method outside our actor-focused multi-agent comparison. All checkpoints are trained for 500 episodes before evaluation. Following~\cite{yun2026loghd,neuralhd}, we post-training quantize each actor's parameters---the action-hypervector matrix $\mathbf{C}$ for VSPG, and all actor weights and biases for DNN and Raw-Linear---to $b\in\{1,2,4,8\}$-bit signed integers using per-tensor min--max affine quantization. We then flip each stored bit independently with probability $p$ and dequantize the corrupted parameters back to float32 before evaluation.

% We evaluate VSPG on SustainGym~\cite{yeh2023sustaingym}, a multi-agent building-control benchmark with noisy, non-stationary physical observations, comparing against Raw-Linear, Linear-VSPG, and a discrete MA-PPO baseline under the same multi-agent formulation; QHD is omitted as a single-agent value-based method outside our actor-focused comparison. All checkpoints are trained for 500 episodes before evaluation. Following ~\cite{yun2026loghd, neuralhd}, we post-training quantize each actor's parameters -- the action-hypervector matrix $\mathbf{C}$ for VSPG, all weight/bias for DNN and Raw-Linear -- to $b \in \{1,2,4,8\}$-bit signed integers via per-tensor min-max affine quantization, then flip each bit independently with probability $p$ and dequantize back to float32 before evaluation.
\subsubsection{Results}
\autoref{tab:sustaingym_results} shows that FHRR- and RFF-VSPG achieve the strongest returns across both climates, whereas Basis-VSPG performs poorly and exhibits high variance. Together with the failure of RFF-VSPG on \texttt{DoorKey-8x8}, this indicates that no encoder is uniformly effective and that VSPG inherits the inductive bias of its encoder-induced similarity. Under post-training corruption, however, \autoref{fig:bitflip} shows that VSPG action memories generally degrade more gracefully than DNN and Raw-Linear actors. Proposition~\ref{prop:bitflip} complements these results for genuinely bipolar memories under direct sign flips, while the affine-quantized real-valued memories are evaluated empirically.
\subsection{Ablation Studies}
We examine two representation-level implications of
Proposition~\ref{prop:kernel}: dimensionality controls how faithfully the
kernel expansion is compressed, while the encoder determines the neighborhood
over which advantage evidence is shared.
\subsubsection{Dimensionality and policy memory.}
\autoref{fig:dim} reports the normalized area under the learning curve from
100-episode-smoothed returns or success rates. On \texttt{CartPole-v1},
performance improves with $D$ and largely saturates beyond $D=1{,}000$.
\texttt{DoorKey-8x8} is more sensitive, but higher dimensions generally
improve learning despite variation across seeds. This trend is consistent
with \eqref{eq:cosine_kernel}: increasing $D$ improves the approximation of
$\kappa(x,x')$ and reduces interference in the compressed policy memory.
\subsubsection{Encoder-Induced Kernel Neighborhoods.}
For trained \texttt{DoorKey-8$\times$8} policies, we retrieve the top-3
neighbors of sampled observations under the fixed encoder-induced similarity
(\autoref{fig:kernel_qualitative}). Basis and FHRR produce coherent
neighborhoods with high kernel similarity, whereas the unsuccessful RFF
configuration yields weaker and less consistent matches. Together with
\autoref{fig:minigrid_control}, these ablations illustrate the mechanism in
Proposition~\ref{prop:kernel}: $D$ governs interference in the compressed
kernel memory, while the encoder determines whether stored advantage evidence
is transferred to observations where it supports useful generalization,
accounting for both the strong and failed configurations.

\vspace{-3mm}
\section{Conclusion}
\vspace{-1mm}
In this paper, we introduced VSPG, a vector-symbolic formulation of
discrete-action policy gradients. Its softmax update becomes
advantage hypervector bundling, while action memories form
fixed-size kernel expansions over experience. For bipolar memories, greedy
action selection is provably robust to random bit flips, and experiments show
favorable sample efficiency, competitive returns, and graceful degradation
under quantization and memory corruption. Future work will test whether these properties extend to vision-driven decision making and robotics~\cite{mnih2013playing,caron2021emerging}.
% In this paper, we introduced VSPG, a vector-symbolic formulation of
% discrete-action policy gradients. We showed that its update corresponds to
% the policy gradient of a linear categorical actor over fixed HDC features,
% expressed as advantage-weighted hypervector bundling, and that the resulting
% action hypervectors can be interpreted as fixed-size kernel memories.
% We further proved that the greedy readout from bipolar action memories tolerates random bit flips with failure probability decaying exponentially in the dimension, and experiments demonstrated favorable sample efficiency, competitive final performance, and small degradation under quantization and bit-flip corruption.
% An
% important next question is whether these advantages extend to
% vision-driven autonomous systems, from classic Atari
% benchmarks~\cite{mnih2013playing} to UAV navigation and control, where the
% policy must operate under visual uncertainty and may depend on a learned
% perception backbone~\cite{caron2021emerging}. Understanding how much of
% VSPG's efficiency and robustness survives this transition remains an open
% direction.

\normalsize
\bibliography{aaai2027}

\newpage
\onecolumn
\appendix
\vspace{-3mm}
\section{Baseline Actor Architectures}
\vspace{-3mm}
\label{sec:supp-baselines}

\textbf{DNN.}
A two-hidden-layer MLP with ReLU activations maps observations to
$|\mathcal{A}|$ action logits, followed by a softmax. The hidden widths
are tuned as described in \autoref{sec:supp-hparam}. For classic
control and MiniGrid, DNN uses the same REINFORCE advantages and Adam
optimizer as VSPG.
On SustainGym, DNN (Multi-Agent) uses decentralized actors and a shared
centralized critic. Each agent maintains its own actor, while a
two-hidden-layer MLP critic (default width 128) maps the concatenated
global state to a scalar value estimate. Training uses GAE
($\lambda{=}0.95$), a PPO-style clipped objective
($\epsilon{=}0.2$), and online normalization of critic targets. The
critic is trained on-policy and discarded at evaluation. Multi-agent
VSPG and Raw-Linear use the same critic, advantages, and clipped
importance-ratio objective, isolating the actor representation rather
than the training procedure. DNN (Single) instead uses one actor and
critic over the global state, with the actor producing a joint
discretized action for all zones under the same training loop.

\textbf{Raw-Linear.}
A single linear layer maps the raw observation directly to
$|\mathcal{A}|$ action logits, followed by a softmax, without hidden
layers or HDC encoding. Since VSPG is also linear in its encoded
features,
$\mathbf{C}\boldsymbol{\phi}(x)$, this baseline isolates the effect of
the hyperdimensional encoding from that of the linear policy. All
remaining training settings, including the SustainGym Multi-Agent and
Single variants, match DNN.

\textbf{QHD.}
QHD~\citep{qhd} is the closest prior HDC baseline for single-agent,
discrete-action RL. It learns a linear hyperdimensional $Q$-function,
$Q(s,a)=M[a]^\top s$, using semi-gradient Q-learning:
\begin{equation}
\label{eq:qhd-update}
Q(s,a) \leftarrow Q(s,a)
+ \beta\left[r+\gamma\max_{a'}Q^{-}(s',a')-Q(s,a)\right]
\end{equation} 
where $Q^{-}$ is a target copy hard-synchronized at the interval reported as \texttt{target} in \autoref{tab:hparam-budget}. For the linear HDC representation, \autoref{eq:qhd-update} becomes $M[a] \mathrel{+}=\beta\left(q_{\mathrm{true}}-q_{\mathrm{pred}}\right)s,$ with $q_{\mathrm{true}}=r+\gamma\max_{a'}Q^{-}(s',a')$ and $q_{\mathrm{pred}}=M[a]^\top s$. Actions are selected $\epsilon$-greedily with linearly decayed exploration. Because QHD is off-policy whereas VSPG is on-policy, we match episode budgets rather than update counts. QHD is omitted from SustainGym because it is a single-agent value-based method and does not fit the actor-focused multi-agent comparison.

\section{Encoder Implementation Details Across Environments}
\label{sec:supp-encoders}

The main paper defines a fixed base map $\varphi_e$ for each encoder family
$e\in\{\mathrm{Basis},\mathrm{FHRR},\mathrm{RFF}\}$. Each
environment-specific encoder is obtained either by applying $\varphi_e$
directly to a preprocessed observation or by composing multiple
$\varphi_e$-encoded components through binding and bundling.

Throughout this section, $\odot$ denotes binding and $\oplus$ denotes
bundling. These symbols refer to the corresponding VSA operations rather than
to one shared scalar operation. For the bipolar Basis encoder, $\odot$ is
element-wise multiplication and $\oplus$ is element-wise addition. For FHRR,
$\odot$ is element-wise complex multiplication, equivalently phase addition,
and $\oplus$ is element-wise complex addition. Every resulting real-valued
hypervector is finally normalized to unit Euclidean norm before being passed
to the actor.

\subsection{MiniGrid: Compositional Encoding over Grid Cells}
\label{sec:supp-encoders-minigrid}

A MiniGrid observation consists of a $7\times7\times3$ partial-view image,
containing an object, color, and state identifier for each cell, together with
the agent direction. Let $\mathcal{V}(x)$ denote the cells marked as visible
in observation $x$, and let $(o_{uv},k_{uv},q_{uv})$ denote the object, color,
and state identifiers at cell $(u,v)$. For the Basis and FHRR encoders, the
raw observation representation has the common compositional form
\begin{equation}
\widetilde{\boldsymbol{\phi}}_{e}(x)
=
\bigoplus_{(u,v)\in\mathcal{V}(x)}
\left[
\varphi_{e}^{\mathrm{pos}}(u,v)
\odot
\varphi_{e}^{\mathrm{obj}}(o_{uv})
\odot
\varphi_{e}^{\mathrm{col}}(k_{uv})
\odot
\varphi_{e}^{\mathrm{st}}(q_{uv})
\right]
\oplus
\varphi_{e}^{\mathrm{dir}}(d),
\qquad
e\in\{\mathrm{Basis},\mathrm{FHRR}\},
\label{eq:supp-minigrid-compositional}
\end{equation}
where $d$ is the current direction. Thus, each cell is represented by binding
its position and categorical components, and the visible-cell representations
are bundled with the direction representation.

\paragraph{Basis.}
MiniGrid uses the sign-thresholded Basis map, so each categorical component is
assigned a fixed bipolar hypervector. Position is represented by binding
independently drawn row and column encodings,
\begin{equation}
\varphi_{\mathrm{Basis}}^{\mathrm{pos}}(u,v)
=
\varphi_{\mathrm{Basis}}^{x}(e_u)
\odot
\varphi_{\mathrm{Basis}}^{y}(e_v),
\label{eq:supp-minigrid-basis-position}
\end{equation}
where $e_u$ and $e_v$ are one-hot identifiers. Object, color, state, and
direction identifiers are encoded in the same way using independent fixed
codebooks. In \autoref{eq:supp-minigrid-compositional}, binding is the
Hadamard product and bundling is vector addition. The resulting vector is
then normalized to unit Euclidean norm.

\paragraph{FHRR.}
For FHRR, the composition in
\autoref{eq:supp-minigrid-compositional} is implemented using unit complex
hypervectors. Let $D_c=D/2$ and define
$\operatorname{cis}(\boldsymbol{\theta})
=\cos(\boldsymbol{\theta})+i\sin(\boldsymbol{\theta})$.
The position encoder uses two Gaussian base-phase vectors,
\begin{equation}
\boldsymbol{\psi}_x,\boldsymbol{\psi}_y
\sim
\mathcal{N}\!\left(
\mathbf{0},
w^{-2}\mathbf{I}_{D_c}
\right),
\qquad
\varphi_{\mathrm{FHRR}}^{\mathrm{pos}}(u,v)
=
\operatorname{cis}\!\left(
u\boldsymbol{\psi}_x+v\boldsymbol{\psi}_y
\right),
\label{eq:supp-minigrid-fhrr-position}
\end{equation}
where $w$ controls the spatial kernel width. Object, color, state, and
direction identifiers use independent fixed phase codebooks,
\begin{equation}
\boldsymbol{\omega}_{r}^{(k)}
\sim
\mathcal{U}(-\pi,\pi)^{D_c},
\qquad
\varphi_{\mathrm{FHRR}}^{r}(k)
=
\operatorname{cis}\!\left(
\boldsymbol{\omega}_{r}^{(k)}
\right),
\quad
r\in\{\mathrm{obj},\mathrm{col},\mathrm{st},\mathrm{dir}\}.
\label{eq:supp-minigrid-fhrr-symbols}
\end{equation}

The representation of visible cell $(u,v)$ is therefore
\begin{equation}
\begin{aligned}
\mathbf{z}_{uv}
&=
\varphi_{\mathrm{FHRR}}^{\mathrm{pos}}(u,v)
\odot
\varphi_{\mathrm{FHRR}}^{\mathrm{obj}}(o_{uv})
\odot
\varphi_{\mathrm{FHRR}}^{\mathrm{col}}(k_{uv})
\odot
\varphi_{\mathrm{FHRR}}^{\mathrm{st}}(q_{uv})
\\
&=
\operatorname{cis}\!\left(
u\boldsymbol{\psi}_x
+
v\boldsymbol{\psi}_y
+
\boldsymbol{\omega}_{\mathrm{obj}}^{(o_{uv})}
+
\boldsymbol{\omega}_{\mathrm{col}}^{(k_{uv})}
+
\boldsymbol{\omega}_{\mathrm{st}}^{(q_{uv})}
\right).
\end{aligned}
\label{eq:supp-minigrid-fhrr-cell}
\end{equation}
Thus, FHRR binding is implemented by element-wise complex multiplication,
which is equivalent to adding the component phases.

The visible-cell hypervectors are bundled by complex addition, after which
the direction hypervector is added:
\begin{equation}
\mathbf{z}(x)
=
\bigoplus_{(u,v)\in\mathcal{V}(x)}
\mathbf{z}_{uv}
\oplus
\varphi_{\mathrm{FHRR}}^{\mathrm{dir}}(d).
\label{eq:supp-minigrid-fhrr-bundle}
\end{equation}
Following the implementation, each complex coordinate is then projected back
to unit magnitude,
\begin{equation}
\widehat{z}_j(x)
=
\frac{z_j(x)}
{\max\{|z_j(x)|,\varepsilon\}},
\qquad
j=1,\ldots,D_c,
\label{eq:supp-minigrid-fhrr-coordinate-normalization}
\end{equation}
before the real and imaginary parts are interleaved into
$\widetilde{\boldsymbol{\phi}}_{\mathrm{FHRR}}(x)\in\mathbb{R}^{D}$.
The interleaving is only a fixed coordinate permutation of the block
real--imaginary representation and therefore preserves inner products. The
real-valued output is finally normalized to unit Euclidean norm.

The Gaussian position phases in
\autoref{eq:supp-minigrid-fhrr-position} induce a smooth spatial kernel. For
a displacement $(\Delta_u,\Delta_v)$,
\begin{equation}
\mathbb{E}\!\left[
\frac{1}{D_c}
\sum_{j=1}^{D_c}
\cos\!\left(
\Delta_u\psi_{x,j}
+
\Delta_v\psi_{y,j}
\right)
\right]
=
\exp\!\left(
-\frac{\Delta_u^2+\Delta_v^2}{2w^2}
\right).
\label{eq:supp-minigrid-fhrr-kernel}
\end{equation}
Hence, $w$ controls the spatial neighborhood over which observations share
representation similarity. Smaller values produce more nearly orthogonal
positions, whereas larger values produce broader spatial generalization. We
use $w=1.0$ unless otherwise stated, giving expected similarity
$\exp(-1/2)\approx0.61$ between horizontally or vertically adjacent cells.

\paragraph{RFF.}
RFF is applied directly rather than compositionally:
\begin{equation}
\widetilde{\boldsymbol{\phi}}_{\mathrm{RFF}}(x)
=
\varphi_{\mathrm{RFF}}\!\left(
T_{\mathrm{MG}}(x)
\right),
\label{eq:supp-minigrid-rff}
\end{equation}
where $T_{\mathrm{MG}}$ flattens the observation into a
$7\times7\times3+1=148$ dimensional vector and rescales the object, color,
state, and direction channels by $10$, $5$, $2$, and $3$, respectively. This
prevents the raw categorical identifier ranges from determining the projection
scale. The output of \autoref{eq:supp-minigrid-rff} is then normalized to unit
Euclidean norm.

\subsection{Classic Control: Direct Encoding}
\label{sec:supp-encoders-classic}

The classic-control environments provide flat continuous observations:
$d=4$ for \texttt{CartPole-v1}, $d=6$ for \texttt{Acrobot-v1}, and
$d=8$ for \texttt{LunarLander-v2}. Since these observations contain no
explicit compositional structure, each encoder is applied directly:
\begin{equation}
\widetilde{\boldsymbol{\phi}}_{e}(x)
=
\varphi_e(x),
\qquad
e\in
\{\mathrm{Basis},\mathrm{FHRR},\mathrm{RFF}\}.
\label{eq:supp-classic-direct-encoding}
\end{equation}
Basis uses the identity nonlinearity,
$\varphi_{\mathrm{Basis}}(x)=\mathbf{W}x$, while FHRR and RFF use their
corresponding fixed random maps from the main paper. The native observation
ranges are already comparable, so no additional range normalization is applied
before \autoref{eq:supp-classic-direct-encoding}. Every encoded observation is
subsequently normalized to unit Euclidean norm.

\subsection{SustainGym: Range-Normalized Direct Encoding}
\label{sec:supp-encoders-sustaingym}

SustainGym provides a $d=10$ flat observation whose components have
substantially different physical scales. We first apply component-wise range
normalization,
\begin{equation}
T_{\mathrm{SG}}(x)
=
2
\frac{x-\mathrm{lo}}
{\mathrm{hi}-\mathrm{lo}}
-1
\in[-1,1]^d,
\label{eq:supp-sustaingym-range-normalization}
\end{equation}
and then apply the corresponding base map directly:
\begin{equation}
\widetilde{\boldsymbol{\phi}}_{e}(x)
=
\varphi_e\!\left(
T_{\mathrm{SG}}(x)
\right),
\qquad
e\in
\{\mathrm{Basis},\mathrm{FHRR},\mathrm{RFF}\}.
\label{eq:supp-sustaingym-direct-encoding}
\end{equation}
The normalization in
\autoref{eq:supp-sustaingym-range-normalization} prevents large-range
quantities such as solar heat gain from dominating the random projection.
FHRR and RFF use their corresponding fixed maps, while Basis uses
$\rho=\operatorname{sign}$. The output of
\autoref{eq:supp-sustaingym-direct-encoding} is finally normalized to unit
Euclidean norm.

\section{Hyperparameter Tuning Budget}
\label{sec:supp-hparam}

\autoref{tab:hparam-budget} reports the search space, number of evaluated
configurations, and selected hyperparameters for each method and environment.
For VSPG, $\tau$ and $\eta$ are tuned jointly because $\tau$ affects both the
logit scale and the effective update magnitude. The baselines receive at least
comparable tuning budgets: DNN and Raw-Linear use similarly sized searches,
while QHD is evaluated over a substantially larger grid. Each selected
configuration is then evaluated over five seeds for the results reported in the
main paper.

\begingroup
\fontsize{7pt}{6pt}\selectfont
\renewcommand{\arraystretch}{1.15}

\setlength{\LTleft}{0pt}
\setlength{\LTright}{0pt}
\setlength{\LTcapwidth}{\textwidth}

\begin{longtable}{
  @{}
  >{\raggedright\arraybackslash}p{.13\textwidth}
  @{\hspace{.015\textwidth}}
  >{\raggedright\arraybackslash}p{.14\textwidth}
  @{\hspace{.015\textwidth}}
  >{\raggedright\arraybackslash}p{.32\textwidth}
  @{\hspace{.015\textwidth}}
  >{\centering\arraybackslash}p{.05\textwidth}
  @{\hspace{.015\textwidth}}
  >{\raggedright\arraybackslash}p{.30\textwidth}
  @{}}
\caption{Hyperparameter search spaces, tuning budgets, and selected
configurations. Each row reports one method--environment configuration
(and encoder for VSPG). Repeated settings are shown so each row is
self-contained; fixed settings are not counted in the \textit{Pts.} column.}
\label{tab:hparam-budget} \\
\toprule
\textbf{Method} & \textbf{Environment} &
\textbf{Search/fixed settings} & \textbf{Pts.} &
\textbf{Chosen best} \\
\midrule
\endfirsthead
\toprule
\textbf{Method} & \textbf{Environment} &
\textbf{Search/fixed settings} & \textbf{Pts.} &
\textbf{Chosen best} \\
\midrule
\endhead
\midrule \multicolumn{5}{r}{\textit{continued on next page}} \\
\endfoot
\bottomrule
\endlastfoot

\multicolumn{5}{@{}l}{\textit{MiniGrid}} \\
VSPG (Basis) & \texttt{Empty-5x5} &
$\tau \in \{1,2,5,10\}$, $\eta \in \{10^{-3},5{\times}10^{-3},10^{-2},5{\times}10^{-2}\}$ &
16 & $\tau{=}10,\eta{=}5{\times}10^{-3}$ \\
VSPG (Basis) & \texttt{DoorKey-5x5} &
$\tau \in \{1,2,5,10\}$, $\eta \in \{10^{-3},5{\times}10^{-3},10^{-2},5{\times}10^{-2}\}$ &
16 & $\tau{=}10,\eta{=}10^{-3}$ \\
VSPG (Basis) & \texttt{DoorKey-8x8} &
$\tau \in \{1,2,5,10\}$, $\eta \in \{10^{-3},5{\times}10^{-3},10^{-2},5{\times}10^{-2}\}$ &
16 & $\tau{=}10,\eta{=}10^{-3}$ \\
VSPG (FHRR) & \texttt{Empty-5x5} &
$\tau \in \{1,2,5,10\}$, $\eta \in \{10^{-3},5{\times}10^{-3},10^{-2},5{\times}10^{-2}\}$, $w{=}1.0$ &
16 & $\tau{=}10,\eta{=}10^{-2}$ \\
VSPG (FHRR) & \texttt{DoorKey-5x5} &
$\tau \in \{1,2,5,10\}$, $\eta \in \{10^{-3},5{\times}10^{-3},10^{-2},5{\times}10^{-2}\}$, $w{=}1.0$ &
16 & $\tau{=}10,\eta{=}10^{-3}$ \\
VSPG (FHRR) & \texttt{DoorKey-8x8} &
$\tau \in \{1,2,5,10\}$, $\eta \in \{10^{-3},5{\times}10^{-3},10^{-2},5{\times}10^{-2}\}$, $w{=}1.0$ &
16 & $\tau{=}10,\eta{=}10^{-3}$ \\
VSPG (RFF) & \texttt{Empty-5x5} &
$\tau \in \{5,10,20\}$, $\eta \in \{5{\times}10^{-3},10^{-2},5{\times}10^{-2}\}$, $\sigma \in \{0.5,1.0\}$ &
18 & $\tau{=}20,\eta{=}10^{-2},\sigma{=}1.0$ \\
VSPG (RFF) & \texttt{DoorKey-5x5} &
$\tau \in \{5,10,20\}$, $\eta \in \{5{\times}10^{-3},10^{-2},5{\times}10^{-2}\}$, $\sigma \in \{0.5,1.0\}$ &
18 & $\tau{=}10,\eta{=}5{\times}10^{-3},\sigma{=}1.0$ \\
VSPG (RFF) & \texttt{DoorKey-8x8} &
$\tau \in \{1,2,5,10\}$, $\eta \in \{10^{-2},5{\times}10^{-2}\}$, $\sigma \in \{0.5,1.0\}$  &
16 & $\tau{=}5,\eta{=}5{\times}10^{-2},\sigma{=}0.5$ \\
DNN & \texttt{Empty-5x5} &
lr $\in \{10^{-4},3{\times}10^{-4},10^{-3},3{\times}10^{-3}\}$, hidden $\in \{[64,32],[128,64],[256,128],[256,256]\}$ &
16 & lr${=}10^{-4}$, hidden${=}[128,64]$ \\
DNN & \texttt{DoorKey-5x5} &
lr $\in \{10^{-4},3{\times}10^{-4},10^{-3},3{\times}10^{-3}\}$, hidden $\in \{[64,32],[128,64],[256,128],[256,256]\}$ &
16 & lr${=}3{\times}10^{-4}$, hidden${=}[256,128]$ \\
DNN & \texttt{DoorKey-8x8} &
lr $\in \{10^{-4},3{\times}10^{-4},10^{-3},3{\times}10^{-3}\}$, hidden $\in \{[64,32],[128,64],[256,128],[256,256]\}$ &
16 & lr${=}3{\times}10^{-4}$, hidden${=}[256,128]$ \\
Raw-Linear & \texttt{Empty-5x5} &
lr $\in \{10^{-4},3{\times}10^{-4},10^{-3},3{\times}10^{-3}\}$, $\tau \in \{0.5,1.0,2.0,5.0\}$ &
16 & lr${=}10^{-3}$, $\tau{=}0.5$ \\
Raw-Linear & \texttt{DoorKey-5x5} &
lr $\in \{10^{-4},3{\times}10^{-4},10^{-3},3{\times}10^{-3}\}$, $\tau \in \{0.5,1.0,2.0,5.0\}$ &
16 & lr${=}10^{-3}$, $\tau{=}5.0$ \\
Raw-Linear & \texttt{DoorKey-8x8} &
lr $\in \{10^{-4},3{\times}10^{-4},10^{-3},3{\times}10^{-3}\}$, $\tau \in \{0.5,1.0,2.0,5.0\}$ &
16 & lr${=}10^{-3}$, $\tau{=}5.0$ \\

\multicolumn{5}{@{}l}{\textit{Classic control}} \\

VSPG (FHRR) & \texttt{CartPole-v1} &
$\tau \in \{10,20,40\}$, $\eta \in \{10^{-6},10^{-5},10^{-4},10^{-3}\}$, advantage${=}$REINFORCE  &
12 & $\tau{=}40,\eta{=}10^{-5}$ \\
VSPG (FHRR) & \texttt{LunarLander-v2} &
$\tau \in \{10,20,40\}$, $\eta \in \{10^{-6},10^{-5},10^{-4},10^{-3}\}$, advantage${=}$REINFORCE  &
12 & $\tau{=}40,\eta{=}10^{-5}$ \\
VSPG (FHRR) & \texttt{Acrobot-v1} &
$\tau \in \{5,7,10,20\}$, $\eta \in \{5{\times}10^{-4},10^{-4},10^{-3}\}$, advantage${=}$GAE+PPO-clip  &
12 & $\tau{=}10,\eta{=}10^{-3}$ \\

VSPG (Basis) & \texttt{CartPole-v1} &
$\tau \in \{10,20,40\}$, $\eta \in \{10^{-6},10^{-5},10^{-4},10^{-3}\}$, advantage${=}$REINFORCE  &
12 & $\tau{=}40,\eta{=}10^{-5}$ \\
VSPG (Basis) & \texttt{LunarLander-v2} &
$\tau \in \{10,20,40\}$, $\eta \in \{10^{-6},10^{-5},10^{-4},10^{-3}\}$, advantage${=}$REINFORCE  &
12 & $\tau{=}40,\eta{=}10^{-5}$ \\
VSPG (Basis) & \texttt{Acrobot-v1} &
$\tau \in \{5,7,10,20\}$, $\eta \in \{5{\times}10^{-4},10^{-4},10^{-3}\}$, advantage${=}$GAE+PPO-clip  &
12 & $\tau{=}10,\eta{=}10^{-3}$  \\

VSPG (RFF) & \texttt{CartPole-v1} &
$\tau \in \{10,20,40\}$, $\eta \in \{10^{-6},10^{-5},10^{-4},10^{-3}\}$, advantage${=}$REINFORCE  &
12 & $\tau{=}40,\eta{=}10^{-5}$ \\

VSPG (RFF) & \texttt{LunarLander-v2} &
$\tau \in \{10,20,40\}$, $\eta \in \{10^{-6},10^{-5},10^{-4},10^{-3}\}$, advantage${=}$REINFORCE  &
12 & $\tau{=}40,\eta{=}10^{-5}$ \\

VSPG (RFF) & \texttt{Acrobot-v1} &
$\tau \in \{5,7,10,20\}$, $\eta \in \{5{\times}10^{-4},10^{-4},10^{-3}\}$, advantage${=}$GAE+PPO-clip  &
12 & $\tau{=}10,\eta{=}10^{-3}$ \\

DNN & \texttt{CartPole-v1} &
lr $\in \{10^{-4},3{\times}10^{-4},10^{-3},3{\times}10^{-3}\}$, hidden $\in \{[64,32],[128,64],[256,128],[256,256]\}$ &
16 & lr${=}3{\times}10^{-4}$, hidden${=}[128,64]$ \\
DNN & \texttt{LunarLander-v2} &
lr $\in \{10^{-4},3{\times}10^{-4},10^{-3},3{\times}10^{-3}\}$, hidden $\in \{[64,32],[128,64],[256,128],[256,256]\}$ &
16 & lr${=}3{\times}10^{-4}$, hidden${=}[128,64]$ \\
DNN & \texttt{Acrobot-v1} &
lr $\in \{10^{-4},3{\times}10^{-4},10^{-3},3{\times}10^{-3}\}$, hidden $\in \{[64,32],[128,64],[256,128],[256,256]\}$ &
16 & lr${=}10^{-4}$, hidden${=}[256,128]$ \\
Raw-Linear & \texttt{CartPole-v1} &
lr $\in \{10^{-4},3{\times}10^{-4},10^{-3},3{\times}10^{-3}\}$, $\tau \in \{0.5,1.0,2.0,5.0\}$ &
16 & lr${=}10^{-3}$, $\tau{=}5.0$ \\
Raw-Linear & \texttt{LunarLander-v2} &
lr $\in \{10^{-4},3{\times}10^{-4},10^{-3},3{\times}10^{-3}\}$, $\tau \in \{0.5,1.0,2.0,5.0\}$ &
16 & lr${=}10^{-3}$, $\tau{=}5.0$ \\
Raw-Linear & \texttt{Acrobot-v1} &
lr $\in \{10^{-4},3{\times}10^{-4},10^{-3},3{\times}10^{-3}\}$, $\tau \in \{0.5,1.0,2.0,5.0\}$ &
16 & lr${=}10^{-3}$, $\tau{=}2.0$ \\

\multicolumn{5}{@{}l}{\textit{SustainGym}} \\
VSPG (Basis) & \texttt{hot\_dry} &
$\tau \in \{0.5,1,2,5\}$, $\eta \in \{10^{-3},5{\times}10^{-3},10^{-2},5{\times}10^{-2}\}$ &
16 & $\tau{=}2,\eta{=}5{\times}10^{-3}$ \\
VSPG (Basis) & \texttt{warm\_humid} &
$\tau \in \{0.5,1,2,5\}$, $\eta \in \{10^{-3},5{\times}10^{-3},10^{-2},5{\times}10^{-2}\}$ &
16 & $\tau{=}1,\eta{=}5{\times}10^{-3}$ \\
VSPG (FHRR) & \texttt{hot\_dry} &
$\tau \in \{0.5,1,2,5\}$, $\eta \in \{10^{-3},5{\times}10^{-3},10^{-2},5{\times}10^{-2}\}$, $w{=}1.0$ &
16 & $\tau{=}5,\eta{=}10^{-3}$ \\
VSPG (FHRR) & \texttt{warm\_humid} &
$\tau \in \{0.5,1,2,5\}$, $\eta \in \{10^{-3},5{\times}10^{-3},10^{-2},5{\times}10^{-2}\}$, $w{=}1.0$ &
16 & $\tau{=}2,\eta{=}5{\times}10^{-3}$ \\
VSPG (RFF) & \texttt{hot\_dry} &
$\tau \in \{0.5,1,2,5\}$, $\eta \in \{10^{-3},5{\times}10^{-3},10^{-2},5{\times}10^{-2}\}$, $\sigma{=}1.0$ &
16 & $\tau{=}2,\eta{=}10^{-2}$ \\
VSPG (RFF) & \texttt{warm\_humid} &
$\tau \in \{0.5,1,2,5\}$, $\eta \in \{10^{-3},5{\times}10^{-3},10^{-2},5{\times}10^{-2}\}$, $\sigma{=}1.0$ &
16 & $\tau{=}1,\eta{=}5{\times}10^{-3}$ \\
DNN & \texttt{hot\_dry} &
lr $\in \{10^{-4},3{\times}10^{-4},10^{-3},3{\times}10^{-3}\}$, hidden $\in \{[64,32],[128,64],[256,128],[256,256]\}$ &
16 & lr${=}10^{-4}$, hidden${=}[128,64]$ \\
DNN & \texttt{warm\_humid} &
lr $\in \{10^{-4},3{\times}10^{-4},10^{-3},3{\times}10^{-3}\}$, hidden $\in \{[64,32],[128,64],[256,128],[256,256]\}$ &
16 & lr${=}10^{-3}$, hidden${=}[256,256]$ \\
Raw-Linear & \texttt{hot\_dry} &
lr $\in \{10^{-4},3{\times}10^{-4},10^{-3},3{\times}10^{-3}\}$, $\tau \in \{0.5,1,2,5\}$ &
16 & lr${=}3{\times}10^{-4}, \tau=2.0$ \\
Raw-Linear & \texttt{warm\_humid} &
lr $\in \{10^{-4},3{\times}10^{-4},10^{-3},3{\times}10^{-3}\}$, $\tau \in \{0.5,1,2,5\}$ &
16 & lr${=}3{\times}10^{-3}, \tau=2.0$ \\

\multicolumn{5}{@{}l}{\textit{Value-based}} \\
QHD & \texttt{CartPole-v1} &
$\beta \in \{10^{-3},5{\times}10^{-3},10^{-2},5{\times}10^{-2}\}$, batch $\in \{2,4,10,32\}$, target $\in \{50,200\}$, buffer $\in \{2000,50000\}$ &
64 & $\beta{=}5{\times}10^{-2}$, batch${=}32$, buffer${=}2000$, target${=}50$ \\
QHD & \texttt{LunarLander-v2} &
$\beta \in \{10^{-3},5{\times}10^{-3},10^{-2},5{\times}10^{-2}\}$, batch $\in \{2,4,10,32\}$, target $\in \{50,200\}$, buffer $\in \{2000,50000\}$ &
64 & $\beta{=}10^{-2}$, batch${=}32$, buffer${=}50000$, target${=}50$ \\
QHD & \texttt{Acrobot-v1} &
$\beta \in \{10^{-3},5{\times}10^{-3},10^{-2},5{\times}10^{-2}\}$, batch $\in \{2,4,10,32\}$, target $\in \{50,200\}$, buffer $\in \{2000,50000\}$ &
64 & $\beta{=}10^{-3}$, batch${=}10$, buffer${=}50000$, target${=}50$ \\
QHD & \texttt{Empty-5x5} &
$\beta \in \{10^{-3},5{\times}10^{-3},10^{-2},5{\times}10^{-2}\}$, batch $\in \{2,4,10,32\}$, target $\in \{50,200\}$, buffer $\in \{2000,50000\}$ &
64 & $\beta{=}10^{-2}$, batch${=}10$, buffer${=}2000$, target${=}200$ \\
QHD & \texttt{DoorKey-5x5} &
$\beta \in \{10^{-3},5{\times}10^{-3},10^{-2},5{\times}10^{-2}\}$, batch $\in \{2,4,10,32\}$, target $\in \{50,200\}$, buffer $\in \{2000,50000\}$ &
64 & $\beta{=}10^{-2}$, batch${=}32$, buffer${=}50000$, target${=}200$ \\
QHD & \texttt{DoorKey-8x8} &
$\beta \in \{10^{-3},5{\times}10^{-3},10^{-2},5{\times}10^{-2}\}$, batch $\in \{2,4,10,32\}$, target $\in \{50,200\}$, buffer $\in \{2000,50000\}$ &
64 & $\beta{=}10^{-2}$, batch${=}32$, buffer${=}50000$, target${=}200$ \\
\end{longtable}
\endgroup

\section{Closed-Form Equivalence and the Sphere Constraint}
\label{sec:supp-normalization}

\begin{figure*}[h!]
\centering
\includegraphics[width=\textwidth]{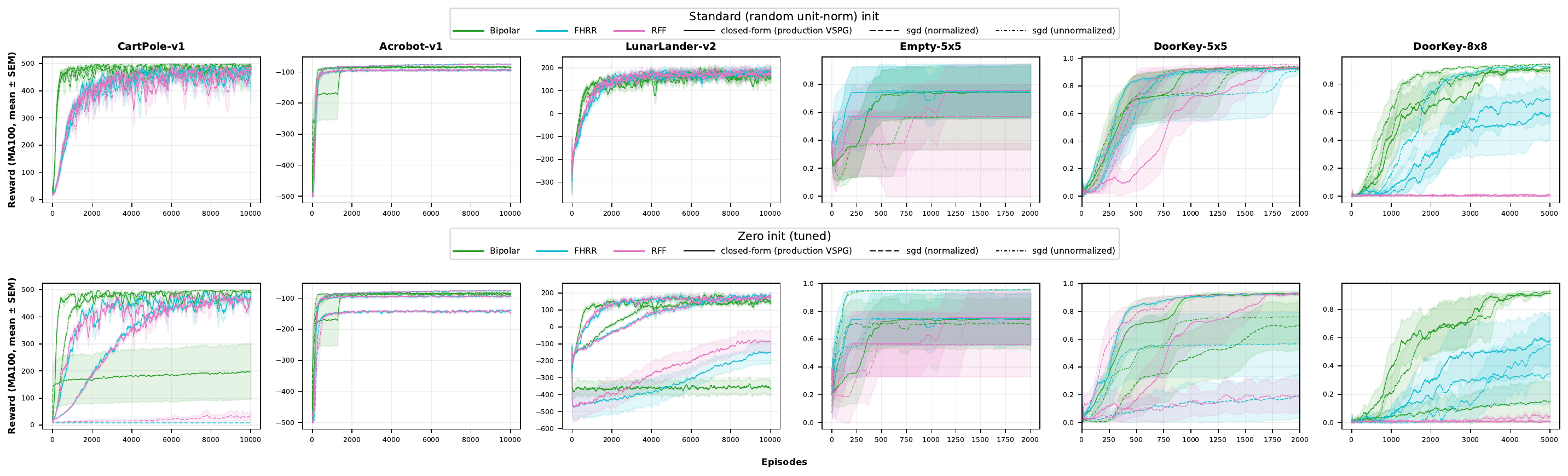}
\caption{Closed-form and directly differentiated VSPG updates under standard
unit-norm initialization (top) and zero initialization (bottom). Colors denote
the encoder; solid, dashed, and dotted lines indicate closed-form, normalized
SGD, and unnormalized SGD, respectively. Shading shows $\pm$ SEM over five
seeds. Standard initialization yields broadly similar behavior, whereas zero
initialization reveals sensitivity to sphere projection.}
\label{fig:row-normalization}
\end{figure*}

VSPG can be viewed both as a vector-symbolic actor trained by
advantage-weighted bundling and as a log-linear softmax policy over fixed HDC
features. Proposition~1 connects these views by showing that the bundling term
$\Lambda^\top S$ is exactly the sampled policy gradient of the softmax
surrogate. We examine whether the closed-form implementation reproduces direct
differentiation of the same objective, and how its random unit-norm
initialization and row-wise sphere projection affect learning.

We compare the closed-form VSPG update with directly differentiated variants
with and without row normalization. Each is evaluated under standard Gaussian
unit-norm initialization and exact zero initialization of the action-hypervector
matrix $\mathbf{C}$. Zero initialization removes the random initialization term
in Proposition~2, leaving action vectors formed entirely from accumulated
policy-gradient evidence. All configurations are independently tuned and
evaluated over five seeds across six environments and three encoders.

Under standard initialization, the differentiated variants generally exhibit
learning behavior similar to the closed-form implementation
(Figure~\ref{fig:row-normalization}, top), consistent with Proposition~1.
Larger differences appear on the harder DoorKey tasks, reflecting projection,
parameter-norm dynamics, and independently tuned optimization scales rather
than a different policy objective.

Zero initialization reveals a stronger interaction with row normalization
(Figure~\ref{fig:row-normalization}, bottom). At $\mathbf{C}=0$, the policy is
initially uniform, and the first update contains only evidence from the initial
trajectories. Immediate normalization maps this update to unit norm regardless
of its magnitude, allowing weak or noisy early evidence to determine a
full-scale action direction. Without normalization, the action-vector norms
instead grow gradually with the accumulated gradient signal, and these variants
continue to learn in several settings where the normalized versions remain
weak.

Thus, neither nonzero initialization nor row normalization is required for the
policy-gradient identity itself. Row normalization enforces the bounded
cosine-policy parameterization and fixed-scale action-memory geometry analyzed
in the main paper, but can amplify early updates when combined with zero
initialization. Overall, the differentiated results support the closed-form
implementation as an algebraic realization of policy-gradient learning over
fixed HDC features, while exposing a practical interaction between
initialization and the sphere constraint.

%%%%%%%%%%%%  Supplementary Figures  %%%%%%%%%%%%
%\clearpage

%%%%%%%%%%%%%%%%   End   %%%%%%%%%%%%%%%%
%\end{multicols}  % Method B for two-column formatting (doesn't play well with line numbers), comment out if using method A
\end{document}